\documentclass[accepted]{uai2025}

\usepackage[american]{babel}
\usepackage[utf8]{inputenc}
\usepackage[T1]{fontenc}
\usepackage{microtype}

\usepackage[round]{natbib}
\usepackage{graphicx}
\usepackage{caption}
\usepackage{subcaption}

\usepackage{siunitx}
\usepackage{booktabs}
\usepackage{multirow}

\usepackage{algorithm}
\usepackage[noend]{algorithmic}

\usepackage{url}
\usepackage{hyperref}
\usepackage{breakurl}
\usepackage{xcolor}
\hypersetup{breaklinks,hidelinks}

\usepackage{amsmath}
\usepackage{amsfonts}
\usepackage{amssymb}
\usepackage{mathtools}
\usepackage{amsthm}
\usepackage{dsfont}
\mathtoolsset{showonlyrefs}

\usepackage{enumitem}

\theoremstyle{plain}
\newtheorem{theorem}{Theorem}[section]

\newtheorem{lemma}[theorem]{Lemma}

\theoremstyle{definition}

\newtheorem{assumption}[theorem]{Assumption}
\theoremstyle{remark}
\newtheorem{remark}[theorem]{Remark}

\DeclareMathOperator{\Om}{\Omega}
\DeclareMathOperator{\prob}{\mathbb{P}}
\DeclareMathOperator{\expected}{\mathbb{E}}
\DeclareMathOperator{\sigmaal}{\mathcal{F}}
\DeclareMathOperator{\D}{\mathcal{D}}
\DeclareMathOperator{\Hy}{\mathcal{H}}
\DeclareMathOperator{\X}{\mathcal{X}}
\DeclareMathOperator{\Y}{\mathcal{Y}}
\DeclareMathOperator{\esS}{\mathcal{S}}
\DeclareMathOperator{\borel}{\mathcal{B}}

\newcommand{\mb}{\mathbf}

\title{Partial-Label Learning with Conformal Candidate Cleaning}

\author{\href{mailto:<tobias.fuchs@kit.edu>?Subject=Partial-Label Learning with Conformal Candidate Cleaning}{Tobias~Fuchs}{}}
\author{Florian~Kalinke}
\affil{Information Systems, Karlsruhe Institute of Technology, Karlsruhe, Germany}

\begin{document}
\maketitle

\begin{abstract}
    Real-world data is often ambiguous; for example, human annotation produces instances with multiple conflicting class labels.
    Partial-label learning (PLL) aims at training a classifier in this challenging setting, where each instance is associated with a set of candidate labels and one correct, but unknown, class label.
    A multitude of algorithms targeting this setting exists and, to enhance their prediction quality, several extensions that are applicable across a wide range of PLL methods have been introduced.
    While many of these extensions rely on heuristics, this article proposes a novel enhancing method that incrementally prunes candidate sets using conformal prediction.
    To work around the missing labeled validation set, which is typically required for conformal prediction, we propose a strategy that alternates between training a PLL classifier to label the validation set, leveraging these predicted class labels for calibration, and pruning candidate labels that are not part of the resulting conformal sets.
    In this sense, our method alternates between empirical risk minimization and candidate set pruning.
    We establish that our pruning method preserves the conformal validity with respect to the unknown ground truth.
    Our extensive experiments on artificial and real-world data show that the proposed approach significantly improves the test set accuracies of several state-of-the-art PLL classifiers.
\end{abstract}

%
%

\section{Introduction}
\label{sec:intro}
Real-world data is often noisy and ambiguous.
In crowdsourcing, for example, different annotators can assign several conflicting class labels to the same instance.
Other examples with ambiguous data include web mining \citep{GuillauminVS10,ZengXJCGXM13} and audio classification \citep{BriggsFR12}.
While such datasets can be manually cleaned, sanitizing data is costly, especially for large-scale datasets.
Partial-label learning (PLL; \citealt{JinG02,LvXF0GS20,XuQGZ21,crosel2024}) provides a principled way of dealing with such conflicting data.
More specifically, in PLL, instances are annotated with sets of candidate labels of which only one unknown label is the correct class label.
PLL permits training a multi-class classifier in this weakly-supervised setting.

Many algorithms targeting the PLL problem exist.
Recently, several extensions \citep{BaoHZ21,BaoHZ22,WangZ22,ZhangWB22,0009LLQG23} that can be combined with a wide range of PLL methods have been proposed, which aim at further improving their predictive performance.
Typically, different PLL classifiers perform best on different datasets.
In this sense, having extensions that are applicable to a multitude of different PLL algorithms is extremely beneficial.
These extensions include feature selection and candidate cleaning techniques, which clean the instance space and candidate label space, respectively.
However, many of these extensions depend on heuristics.

In contrast, this article proposes a novel method that alternates between training a PLL classifier through empirical risk minimization and pruning the candidate sets using conformal prediction,
which output sets of possible labels that contain the correct label with a specified confidence level \citep{lei14,sadinle19}.
In our pruning step, we remove candidate labels if they are not part of these predicted conformal sets.
This principled way of reducing the candidate set ambiguity benefits the training of the PLL classifier when compared to the existing heuristic thresholds.
Our extension significantly improves the prediction quality of several state-of-the-art PLL methods across a variety of datasets and experimental settings.
To guarantee the validity of the conformal classifier used in the pruning step, one usually requires a labeled validation set for the calibration of the coverage guarantee.
In the PLL setting, however, ground truth is unavailable.
To resolve this serious issue, we propose a strategy that trains a PLL classifier, uses its predictions to label the validation set, calibrates the conformal sets with the validation set, and prunes candidate labels that are not part of these conformal sets.
We show that our method preserves the conformal validity with respect to the unknown ground truth.

Our \textbf{contributions} can be summarized as follows.
\begin{itemize}
    \item \emph{Algorithm.} We propose a novel candidate cleaning method that alternates between training a PLL classifier and pruning the PLL candidate sets.
          Our algorithm significantly improves the predictive performance of several state-of-the-art PLL approaches.
    \item \emph{Experiments.} Extensive experiments on artificial and real-world partially labeled data support our claims.
          An ablation study further demonstrates the usefulness of the proposed strategy.
          We make our source code and data openly available at {\small \url{https://github.com/mathefuchs/pll-with-conformal-candidate-cleaning}}.
    \item \emph{Theoretical analysis.} We analyze our method and show that the pruning step yields valid conformal sets.
\end{itemize}

\paragraph{Structure of the paper.}
Section~\ref{sec:notations} establishes our notations and states the partial-label learning problem,
Section~\ref{sec:rw} discusses related work,
Section~\ref{sec:main} details our contributions,
and Section~\ref{sec:exp} shows our experimental setup and results.
All proofs are deferred to Appendix~\ref{sec:proofs}.
Appendix~\ref{sec:hypar} lists all hyperparameters used within our experiments in detail and
Appendix~\ref{sec:more-exp} contains additional experiments.

\section{Notations}
\label{sec:notations}
This section establishes notations used throughout our work as well as states the partial-label learning problem.

Given a $d$-dimensional real-valued feature space $\X = \mathbb{R}^d$ and a set $\Y = [k] := \{ 1, \dots, k \}$ of $3 \leq k \in \mathbb{N}$ classes, a partially-labeled training dataset $\D = \{ (x_i, s_i) \in \X \times 2^{\Y} : i \in [n] \}$ contains $n$ training instances with associated feature vectors $x_i \in \X$ and candidate labels $s_i \subseteq \Y$ for each $i \in [n]$.
Their respective ground-truth labels $y_i \in \Y$ are unknown during training, but $y_i \in s_i$.
We split the dataset $\D$ into a training set $\D_t$ and a dataset $\D_v$ for calibration.

Let $\Om = \X \times \Y \times 2^{\Y}$.
Underlying partial-label learning (PLL) is the probability triplet $(\Om, \borel( \Om ), \prob)$ with $\borel$ denoting the Borel $\sigma$-algebra.
We denote by $X : \Om \to \X$, $Y : \Om \to \Y$, and $S : \Om \to 2^{\Y}$ the random variables governing the occurrence of an instance's features, ground-truth label, and its candidate labels, respectively.
Their realizations are denoted by $x_i$, $y_i$, and $s_i$.
We denote by $\prob_X$ the marginal and by $\prob_{XY}$ and $\prob_{XS}$ the joint distribution of $(X, Y)$ and $(X, S)$, respectively.
$\prob_{XY}$ coincides with the probability measure usually underlying the supervised setting.
We denote with $\prob_n := \prob_{XS}^n$ the $n$-fold product of $\prob_{XS}$.
The cumulative distribution function of the random variable $X$ is $F_X(t) = \prob_X(X \leq t)$ and its empirical counterpart is $\hat{F}_X(t) = \frac{1}{n} \sum_{i=1}^n \mathds{1}_{\{ X_i \leq t \}}$, where $X_i, \dots, X_n \overset{i.i.d.}{\sim} \prob_X$.

Let $\ell : [0, 1]^k \times \Y \to \mathbb{R}_{\geq 0}$ denote a measurable loss function, e.g., the log-loss.
PLL aims to train a probabilistic classifier $f : \X \to [0, 1]^k$ with $\sum_{j = 1}^k f_j(x) = 1$, for $x \in \X$, that minimizes the risk $R(f) = \expected_{XS}[ \sum_{y = 1}^{k} W_{X,S,y} \, \ell(f(X), y) ]$, where $W_{X,S,y}$ are label weights to control the influence of different loss terms.
$f_{y}(x)$ denotes the $y$-th entry of the vector $f(x) \in [0, 1]^k$.

Common instantiations for $W_{X,S,y}$ include the average strategy $W_{X,S,y}^{(\operatorname{avg})} = \mathds{1}_{\{y \in S\}} / |S|$ \citep{HullermeierB06,CourST11} and the minimum strategy
\begin{equation}
    W_{X,S,y}^{(\operatorname{min})} = \frac{\prob_{Y \mid X}(Y = y)}{\sum_{j \in S} \prob_{Y \mid X}(Y = j)}
    \label{eq:min-loss}
\end{equation}
\citep{LvXF0GS20,FengL0X0G0S20},
which weights the loss based on the relevancy of each label.

For the minimum strategy in~\eqref{eq:min-loss}, the true risk takes the form
\begin{equation}
    \hspace{-0.01cm} R(f) = \expected_{XS}\!\big[ \sum_{y = 1}^k \frac{\prob_{Y \mid X}(Y = y)}{\sum_{j \in S} \prob_{Y \mid X}(Y = j)} \ell(f(X), y) \big] \text{.}
    \label{eq:true-risk}
\end{equation}
The empirical version of the risk is obtained by substituting the expectation with a sample mean:
\begin{align}
    \hat{R}(f) = \frac{1}{n} \sum_{i = 1}^{n} \sum_{y = 1}^k w_{ij} \ell(f(x_i), y) \text{,} \label{eq:erm}
\end{align}
where $(x_i, s_i) \in \D$ and $w_{ij} \in [0, 1]$ approximates the label relevancy $W_{X,S,y}^{(\operatorname{min})}$ in~\eqref{eq:min-loss} using
\begin{equation}
    w_{ij} = \begin{cases}
        f_j(x_i) / \sum_{j' \in s_i} f_{j'}(x_i) & \text{if $j \in s_i$,} \\
        0                                        & \text{else,}
    \end{cases}
    \label{eq:label-update}
\end{equation}
using a trained classifier $f : \X \to [0, 1]^k$.

Let $\Hy = \{ f : \X \to [0, 1]^k \mid f \text{ measurable, } \forall x \in \X : \sum_{j=1}^k f_j(x) = 1 \}$ denote the hypothesis space, $f^* = \arg\min_{f \in \Hy} R(f)$ the true risk minimizer, and $\hat{f} = \arg\min_{f \in \Hy} \hat{R}(f)$ the empirical risk minimizer.
An optimal multi-class classifier must be of the form $f^*_y(x) = \prob_{Y \mid X = x}(Y = y)$ \citep[Lemma~1]{YuLGT18}.
We make the common assumption that the hypothesis space $\Hy$ is well-specified, that is, $f^* \in \Hy$ \citep{tsybakov2004optimal,ErvenGMRW15}.
The class label of each instance $x \in \X$ with the highest probabilistic prediction, that is, $\hat{y}_x = \arg\max_{y \in \Y} \hat{f}_y(x)$, is called \emph{pseudo-label}.

\section{Existing Work}
\label{sec:rw}
Partial-label learning is one out of many weakly-supervised learning frameworks \citep{Bylander94,HadyS13,IshidaNMS19}, where training instances are annotated with multiple candidate labels.
Section~\ref{sec:rw-pll} discusses related work regarding partial-label learning and
Section~\ref{sec:rw-sets} discusses related work regarding set-valued prediction-making, which is a natural fit for representing the ambiguity of the PLL candidate sets.

\subsection{Partial-Label Learning (PLL)}
\label{sec:rw-pll}
PLL is a weakly-supervised learning problem that has gained significant attention over the last decades.
Most approaches adapt common supervised classification algorithms to the PLL setting.
Examples include a logistic regression formulation \citep{grandvalet2002logistic},
expectation-maximization strategies \citep{JinG02,LiuD12},
nearest-neighbor methods \citep{HullermeierB06,ZhangY15a,fuchs2024partiallabel},
support-vector classifiers \citep{NguyenC08,CourST11,YuZ17},
custom stacking and boosting ensembles \citep{ZhangYT17,TangZ17,WuZ18},
and label propagation strategies \citep{ZhangY15a,ZhangZL16,XuLG19,WangLZ19,FengA19}.

Recent state-of-the-art methods \citep{LvXF0GS20,FengL0X0G0S20,XuQGZ21,ZhangF0L0QS22,WangXLF0CZ22,0009LLQG23,crosel2024} minimize variations of~\eqref{eq:erm} with the weights as in~\eqref{eq:label-update} using different deep learning approaches.
The minimum loss reweighs the loss terms to only include the most likely class labels.
\citet{GongB024} extend this idea by introducing a smoothing component.

\citet{LvXF0GS20,FengL0X0G0S20} iteratively refine the PLL candidate sets by alternating between training a model $f : \X \to [0, 1]^k$ using empirical risk minimization on~\eqref{eq:erm} and updating the label weights $w_{ij}$ in~\eqref{eq:label-update} using the trained classifier $f$.
At the beginning, the weights $w_{ij}$ are initialized with uniform weights on the respective candidate sets: $w_{ij} = 1/|s_i|$ if $j \in s_i$, else $0$, which coincides with the average strategy \citep{HullermeierB06,CourST11}.
They further show that the resulting classifier is risk consistent with the Bayes classifier $f^*$, if the small-ambiguity-degree condition holds \citep{CourST11,LiuD12}.
The condition requires that there is no incorrect label $\bar{y} \neq y$, which co-occurs with the correct label $y$ in a candidate set with a probability of one.
Formally, one imposes that
$\sup_{\substack{x \in \X, y \in \Y, \bar{y} \in \Y, \bar{y} \neq y}} \prob_{S \mid X = x, Y = y}( \bar{y} \in S ) < 1$.

Because of the huge variety of PLL methods, there are recent algorithms that can be combined with any of the above to improve prediction performance further.
\citet{WangZ22} propose a feature augmentation technique based on class prototypes and \citet{BaoHZ21,BaoHZ22,ZhangWB22} propose feature selection strategies for PLL data.
Existing state-of-the-art methods achieve significantly better accuracies when trained on these modified feature sets.

\citet{0009LLQG23} propose the method \textsc{Pop}, which gradually removes unlikely class labels from the candidate sets if the margin between the most likely and the second-most likely class label exceeds some heuristic threshold.
In contrast, our method gradually removes unlikely class labels based on the set-valued conformal prediction framework, which provides a more principled way of cleaning the candidate sets.
Our method significantly improves the test set accuracies of several state-of-the-art methods including the method \textsc{Pop}.

\subsection{Set-Valued Predictions}
\label{sec:rw-sets}
Recent methods in supervised multi-class classification \citep{lei14,barber2023conformal,MozannarLWSDS23,MaoM024,NarasimhanMJGK24} explore training set-valued predictors $C : \X \to 2^{\Y}$ rather than single-label classifiers $f : \X \to \Y$ as they offer more flexibility in representing the uncertainty involved in prediction-making.
Set-valued prediction-making involves a variety of problem formulations including reject options and conformal prediction.
Reject options allow one to abstain from individual predictions if unsure alleviating the cost of misclassifications;
see \citet{fuchs2024partiallabel} for a recent study of reject options in PLL.

In conformal prediction, classifiers output sets of class labels $C(x) \subseteq \Y$.
\emph{Valid} conformal predictors guarantee that
\begin{align}
    \prob_{XY}(Y \in C(X)) \geq 1 - \alpha \text{,}
    \label{eq:valid}
\end{align}
which means that the correct label is part of a conformal set with a given error level of at most $\alpha \in (0, 1)$.
The conformal predictor $C$ that outputs $C(x) = \Y$, for $x \in \X$, is trivially valid as it covers the correct label with a probability of one.
To avoid this case, one searches for conformal predictors $C$ with minimal expected cardinality $\expected_{X} |C(X)|$, while still being valid.
In the supervised setting, this is captured by the following optimization problem~\citep{sadinle19}:
\begin{align}
    \min_{C : \X \to 2^{\Y}}\, & \expected_{X} |C(X)| \text{,}   \label{eq:opt-supervised} \\
    \text{subject to}\quad     & \prob_{XY}( Y \in C(X) ) \geq 1 - \alpha \text{.}
\end{align}
Optimal solutions to~\eqref{eq:opt-supervised} are of the form $C(x) = \{ y \in \Y : \prob_{Y \mid X = x}(Y = y) \geq t_\alpha \}$, for $x \in \X$, where $t_\alpha$ is set to
\begin{align}
    t_\alpha = \sup \big\{ & t \in [0, 1] : \prob_{XY}\big[ (x, y) : \\ & \quad \prob_{Y \mid X = x}(Y = y) \geq t \big] \geq 1 - \alpha \big\} \text{,} \quad
    \label{eq:thresh}
\end{align}
where we assume that the quantile function of $\prob_{Y \mid X = x}(Y = y)$ is continuous at $t_\alpha$.\footnote{See \citet[Theorem~1]{sadinle19} for the general case.}
In practice, one approximates $t_\alpha$ by computing the empirical distribution function on a hold-out validation set.
One splits the dataset $\D$ into a dataset $\D_t$ for model training and $\D_v$ for calibrating the conformal predictor $C$ with respect to the confidence level $\alpha$.
The validation set $\D_v$ is assumed to be exchangeable with respect to the joint distribution $\prob_{XY}$.

Conformal prediction is also a natural fit to partial-label learning as both deal with sets of class labels.
\citet{JavanmardiSHH23} examine different ways of achieving valid conformal sets in the PLL context.
However, they do not propose any new PLL method against which we can compare.
Rather, they analyze the properties of different non-conformity measures in this context.
In contrast, our focus is on constructing new PLL methods and evaluating them.

In the following section, we propose a novel candidate cleaning method that is based on conformal prediction and adapts~\eqref{eq:opt-supervised} to the PLL setting to yield valid conformal sets.
The optimization problem~\eqref{eq:opt-supervised} cannot directly be transferred to the PLL context as ground truth for the calibration of the validity property is unavailable.
We propose a strategy that uses the PLL classifier $f$ to label the validation set and then leverages these pseudo-labels for calibration.
We show that this preserves the validity with respect to the ground truth.

\section{PLL with Conformal Cleaning}
\label{sec:main}
We propose a novel candidate cleaning strategy that iteratively cleans the candidate sets of the PLL dataset $\D$ by reducing the candidate set cardinalities.
Our method alternates between training a PLL classifier through empirical risk minimization and pruning the candidate sets based on conformal prediction.
Conformal predictors $C : \X \to 2^{\Y}$ cover the correct label $y_i$ of instance $x_i$ with a specified probability; see~\eqref{eq:valid}.
This coverage property is calibrated using a separate validation set of exchangeable PLL data points that are labeled using the trained PLL algorithm.
As the classifier can give wrong predictions, however, we propose a novel correction strategy that accounts for possible misclassifications when calibrating the coverage of the correct labels against the validation set, which maintains the validity guarantee.
We remove class labels from the candidate sets $s_i$ if they are not part of the predicted conformal set $C(x_i)$ since the correct label $y_i$ is in $C(x_i)$ with a specified confidence level.

This procedure iteratively removes noise from incorrect candidate labels, which benefits the training of the PLL classifier by having to account for less and less noise in each training step.
Many PLL algorithms \citep{LvXF0GS20,0009LLQG23,crosel2024} proceed in a similar manner.
They have in common that they alternate between training a PLL classifier and using its predictions to refine the candidate label weights.
This can equivalently be expressed from an expectation-maximization perspective \citep[Section~5]{WangXLF0CZ22}.
These label propagation strategies are state-of-the-art in many weakly-supervised learning domains.
In contrast to the existing heuristic update rules, however, our proposed method provides a principled way of iteratively cleaning the candidate sets using conformal predictors $C$.

In the following sections, we discuss our method in detail.
Section~\ref{sec:pll-validity} elaborates on the notion of conformal validity in the PLL context,
Section~\ref{sec:correction} details how to correct for the ambiguity in PLL compared to the supervised setting,
Section~\ref{sec:our-algorithm} outlines the proposed algorithm,
Section~\ref{sec:runtime} discusses the method's runtime complexity,
and Section~\ref{sec:connection-pop} discusses the placement of our method with respect to related work.

\subsection{PLL Validity}
\label{sec:pll-validity}
Since we use the conformal predictions $C(x_i)$ to clean the associated candidate sets $s_i$, for $(x_i, s_i) \in \D$, we require that $s_i \cap C(x_i)$ is nonempty with a specified confidence level as otherwise $C(x_i)$ does not contain the unknown correct label $y_i$.
Hence, we adapt~\eqref{eq:valid} to our setting and consider a conformal classifier $C$ valid with respect to the PLL candidate sets if it holds that
\begin{equation}
    \prob_{XS}( S \cap C(X) \neq \emptyset ) \geq 1 - \alpha \text{,}
    \label{eq:pll-valid}
\end{equation}
for a given error level $\alpha \in (0, 1)$.
In other words, conformal predictions $C(x_i)$ need to cover the observed ambiguously labeled candidate sets $s_i$ with a specified probability.
Recall that $C(x) = \Y$, for $x \in \X$, trivially satisfies~\eqref{eq:pll-valid}.
One therefore also wants to minimize the cardinalities $|C(x)|$.
Given the standard PLL assumption that the correct label $y_i$ is within the respective candidate set $s_i$, which implies that $\prob_{S \mid X = x, Y = y}(y \in S) = 1$ for any $(x, y) \in \X \times \Y$, an optimal solution to~\eqref{eq:opt-supervised} is also valid in the sense of~\eqref{eq:pll-valid}.
Theorem~\ref{prop:pll-valid} captures this relationship and underpins our proposed cleaning method, which we detail in Section~\ref{sec:our-algorithm}.

\begin{theorem}
    \label{prop:pll-valid}
    Assume that $\prob_{S \mid X = x, Y = y}(y \in S) = 1$, for any $(x, y) \in \X \times \Y$, and $\alpha \in (0, 1)$.
    Then, an optimal solution $C$ of~\eqref{eq:opt-supervised} satisfies~\eqref{eq:pll-valid}: $\prob_{XS}( S \cap C(X) \neq \emptyset ) \geq 1 - \alpha$.
\end{theorem}

\subsection{Correcting for Misclassification}
\label{sec:correction}
Recall that, in the PLL setting, the ground-truth labels $y$ are unavailable during training, which hinders the approximation of~\eqref{eq:thresh} needed for the solution of~\eqref{eq:opt-supervised}.
Because a solution to~\eqref{eq:opt-supervised} is, however, also desirable in the PLL setting (Theorem~\ref{prop:pll-valid}), we make use of existing PLL algorithms to generate pseudo-labels.
This strategy iteratively learns a prediction model $f : \X \to [0, 1]^k$ that minimizes the empirical risk in~\eqref{eq:erm}.
We use the trained model $f$ to predict the labels on the validation set $\D_v$, which in turn is used for the calibration of the validity guarantee.
Notably, this strategy results in a valid conformal predictor (Theorem~\ref{thm:main}).
We note that it remains open to establish the minimality of the resulting conformal sets (analogous to solutions of~\eqref{eq:opt-supervised}).

At first glance, it might be counter-intuitive to use the trained model $f$ to label the validation set and build conformal sets based on it.
However, we want to recall that the used base PLL classifier is risk consistent \citep[Theorem~4]{FengL0X0G0S20}.
With this result and additional mild assumptions, we can prove that the PLL classifier's predictions cannot be arbitrarily bad (Lemma~\ref{lemma:main}) and, leveraging this, that our conformal predictor is valid for some adapted threshold and error level (Theorem~\ref{thm:main}).

One of our central assumptions is a Bernstein condition \citep{Audibert04,Bartlett06,GrunwaldM20} on the loss difference (Assumption~\ref{def:bernstein}).
The Bernstein condition is defined as follows.

\begin{assumption}[Bernstein Condition]
    \label{def:bernstein}
    Let $B > 0$, $\beta \in (0, 1]$, $f^* = \arg\min_{f \in \Hy} R(f)$ the true risk minimizer, and $\ell : [0, 1]^k \times \Y \to \mathbb{R}_{\geq 0}$ a loss function.
    We assume that the excess loss $L_f(x, y) := \ell(f(x), y) - \ell(f^*(x), y)$ satisfies the $(\beta, B)$-Bernstein condition, that is, for $f \in \Hy$,
    \begin{align}
        \expected_{XY}\!\left[ L_f(X, Y)^2 \right] \leq B \left( \expected_{XY}\!\left[ L_f(X, Y) \right] \right)^\beta \text{.}
    \end{align}
\end{assumption}

Assumption~\ref{def:bernstein} is frequently made in ERM as it allows controlling the variance of the resulting losses, since $\operatorname{Var}_{XY}[ \ell(f(X), Y) - \ell(f^*(X), Y) ] \leq \expected_{XY}[ (\ell(f(X), Y) - \ell(f^*(X), Y))^2 ]$.
In other words, the tail of the distribution of the excess loss must be well-behaved.

Building upon Assumption~\ref{def:bernstein}, we prove the results in the following Lemma~\ref{lemma:main},
which are the main building blocks underlying the proof of our main result.

\begin{lemma}
    \label{lemma:main}
    Let $\hat{f} = \arg\min_{f \in \Hy} \hat{R}(f)$ the empirical risk minimizer,
    $f^* = \arg\min_{f \in \Hy} R(f)$ the true risk minimizer,
    $\hat{y}_x = \arg\max_{y} \hat{f}_y(x)$,
    $y^*_x = \arg\max_{y} f^*_y(x)$,
    and Assumption~\ref{def:bernstein} hold for the excess loss $L_{\hat{f}}$.
    \begin{enumerate}
        \item[(i)] Then, for any $\delta_1 \in (0, 1)$ and some constant $M_1 > 0$, it holds, with $\prob_n$-probability at least $1 - \delta_1$, that
              \begin{align}
                  \expected_{XY}\!\left[ | \hat{f}_Y(X) - f^*_Y(X) | \right]
                  \leq M_1 \left( \frac{\log(1 / \delta_1)}{n} \right)^{\frac{1}{4}\beta} \text{,} \notag
              \end{align}
              assuming that $\ell : (0, 1]^k \times \Y \to \mathbb{R}_{\geq 0}$, $(p, y) \mapsto -\log p_y$ is the log-loss.

        \item[(ii)] Also, for any $\delta_2 \in (0, 1)$ and some constant $M_2 > 0$, it holds, with $\prob_n$-probability at least $1 - \delta_2$, that
              \begin{align}
                  \prob_X\!\left[ \hat{y}_X \neq y^*_X \right]
                  \leq M_2 \left( \frac{\log(1 / \delta_2)}{n} \right)^{\frac{1}{2}\beta} \text{,}
              \end{align}
              given that, for any $x \in \X$ and some constant $\delta_5 \in [0, 1)$,
              $
                  \prob_{Y \mid X = x}\left( Y \in \{ \hat{y}_x, y^*_x \} \right) \geq 1 - \delta_5
                  \text{.}
              $
    \end{enumerate}
\end{lemma}

Intuitively, Lemma~\ref{lemma:main}~$(i)$ and $(ii)$ state that, under mild assumptions, a consistent PLL classifier cannot, in expectation, provide arbitrarily bad predictions.
More precisely, Lemma~\ref{lemma:main}~$(i)$ states that the expected absolute difference in the probabilistic predictions of the empirical and true risk minimizer are upper-bounded.
Lemma~\ref{lemma:main}~$(ii)$ states that the probability of class label predictions of the empirical and true risk minimizer not matching is upper-bounded.
Note that, for $n \to \infty$, both upper-bounds tend to zero.

In the following, we comment on the assumptions made.
Lemma~\ref{lemma:main}~$(i)$ requires the loss function to be the log-loss as it is a local proper loss function \citep{gneiting2007strictly}, that is, $\ell$ is a proper loss function that only uses the $y$-th entry of the vector $p$ in the computation of $\ell(p, y)$, which we use in our proof.
Lemma~\ref{lemma:main}~$(ii)$ requires that the correct label $y^*_x$ and pseudo-label $\hat{y}_x$ have some lower-bound for their conditional probability mass.
Intuitively, the assumption captures that the true class posterior of the correct label $y^*_x$ must have a probability mass that is not arbitrarily close to zero.

Based on the upper bounds in Lemma~\ref{lemma:main}, one can adapt the threshold and confidence levels in~\eqref{eq:opt-supervised} and~\eqref{eq:thresh} such that the conformal guarantee is still valid when using the pseudo-labels on the validation set.
Theorem~\ref{thm:main} states this result.

\begin{theorem}
    \label{thm:main}
    Assume the setting of Lemma~\ref{lemma:main}~$(i)$ and $(ii)$
    and, for any $\delta_6 \in (0, 1)$, $\prob_{Y \mid X = x}(Y = y^*_X) \geq 1 - \delta_6$
    with $y^*_X = \arg\max_{y' \in \Y} f^*_y(X)$.
    For any $\alpha \in (0, 1)$, let
    \begin{align}
        t_\alpha = \sup\{ t \in [0, 1] \mid \hat{F}_{\hat{f}_{\hat{y}_X}(X)}(t) \leq \alpha \} \text{,} \label{eq:thresh-thm}
    \end{align}
    with $\hat{y}_x = \arg\max_{y \in \Y} \hat{f}_y(x)$.
    Then, the conformal set
    \begin{equation}
        C(x) = \{ y \in \Y \mid \hat{f}_{y}(x) \geq t_\alpha - \delta_3 \} \label{eq:conf-pred}
    \end{equation}
    is valid, that is, $\prob_{X}(y^*_X \in C(X)) \geq 1 - \alpha'_n$ holds
    with a $\prob_n$-probability of at least $1 - (\delta_1 + \delta_2 + \delta_4)$, where,
    for any $\delta_1, \delta_2, \delta_4 \in (0, 1)$ and some constants $\beta \in (0, 1]$, $B, \delta_3, M_1, M_2 > 0$,
    \begin{align}
        \alpha'_n & :=
        \frac{1}{\delta_3 (1 - \delta_6)} M_1 \left( \frac{\log(1/\delta_1)}{n} \right)^{\frac{1}{4}\beta} \\
                  & \quad + M_2 \left( \frac{\log(1/\delta_2)}{n} \right)^{\frac{1}{2}\beta}
        + \alpha
        + \left( \frac{\log(2/\delta_4)}{2n} \right)^{\frac{1}{2}} \text{.}
    \end{align}
\end{theorem}

Intuitively, the tighter the upper bounds in Lemma~\ref{lemma:main} are, the smaller the necessary correction of the threshold and confidence level in Theorem~\ref{thm:main}.
In other words, loose upper bounds in Lemma~\ref{lemma:main} lead to high cardinalities of $C(x)$ in~\eqref{eq:conf-pred}.
In contrast, tight upper bounds in Lemma~\ref{lemma:main} lead to small cardinalities of $C(x)$ in~\eqref{eq:conf-pred}.
The following Remark~\ref{remark:thm} details how to obtain conformal validity for a fixed error level.

\begin{remark}
    \label{remark:thm}
    Alternatively, one obtains a fixed error level $\alpha_2 \in (0, 1)$ in Theorem~\ref{thm:main},
    that is, $\prob_{X}(y^*_X \in C(X)) \geq 1 - \alpha_2$, by using
    \begin{align}
        \alpha'' & = \alpha_2 - \frac{1}{\delta_3 (1 - \delta_6)} M_1 \left( \frac{\log(1/\delta_1)}{n} \right)^{\frac{1}{4}\beta} \\
                 & \quad- M_2 \left( \frac{\log(1/\delta_2)}{n} \right)^{\frac{1}{2}\beta}
        - \left(\frac{\log(2/\delta_4)}{2n}\right)^{\frac{1}{2}}
        \text{,} \quad \label{eq:corrected-conf}
    \end{align}
    in the computation of the threshold $t_{\alpha''}$ in~\eqref{eq:thresh-thm}.
\end{remark}

While Remark~\ref{remark:thm} follows from a simple substitution, it explicitly links Theorem~\ref{thm:main} to the setting usually considered in conformal prediction:
One wants to have a conformal predictor that is valid regarding some specified confidence level $\alpha_2$, which Remark~\ref{remark:thm} achieves by using an altered $\alpha''$ in the computation of $t_{\alpha''}$.
If $\alpha'' \leq 0$, the resulting conformal predictor defaults to $C(x) = \Y$, for $x \in \X$, which is valid.
In contrast, given some confidence level $\alpha$, Theorem~\ref{thm:main} gives a conformal predictor that is valid with the confidence level $\alpha'_n \neq \alpha$.

Theorem~\ref{thm:main} enables our proposed algorithm.
When using a consistent PLL classifier to label the validation set, a conformal predictor with a threshold set based on these pseudo-labels still satisfies a conformal validity guarantee for an adapted threshold and error level.
The subsequent section discusses our approach.

\subsection{Proposed Algorithm}
\label{sec:our-algorithm}
Based on the conformal predictor in Theorem~\ref{thm:main}, we propose a novel candidate cleaning strategy that alternates between training a neural-network-based PLL classifier and pruning the candidate labels by conformal prediction.
We outline our method in Algorithm~\ref{alg:method}.
In the following, we provide an overview.
Thereafter, we discuss all parts in detail.

First, we randomly partition the dataset $\D$ into $\D_t$ for training the model and $\D_v$ for calibrating the conformal predictor $C$ based on the current state of the prediction model $f$ (Line~1).
The training set consists of \SI{80}{\%} and the validation set of \SI{20}{\%} of all instances.
We initialize the model $f$ and the label weights $w_{ij}$ in Lines~3--4.
Lines~5--23 contain the main training loop, which can be divided into four phases: (1) Updating the predictions on the validation set $\D_v$ for calibration (Lines~6--7), (2) updating the model's weights $\theta$ through back-propagation (Lines~8--10), (3) cleaning the candidate sets $s_i$ based on the predicted conformal sets $C(x_i)$ (Lines~11--20), and (4) updating the label weights $w_{ij}$ (Lines~21--22).
We detail these phases in the following.

\begin{figure}[t!]
    \vspace{-0.8em}
    \begin{algorithm}[H]
        \caption{Conformal Candidate Cleaning}
        \label{alg:method}
        \begin{algorithmic}[1]
            \vspace{0.1em}
            \item[\textbf{Input:}] PLL dataset $\D = \lbrace (x_i, s_i) \in \X \times 2^{\Y} : i \in [ n ] \rbrace$; conformal error level $\alpha \in (0, 1)$; number of epochs $R$; number of warm-up rounds $R_{\operatorname{warmup}}$;
            \item[\textbf{Output:}] Predictor $f : \X \to [0, 1]^k$, $\sum_{y \in \Y} f_{y}(x) = 1$;
            \vspace{0.1em}
            \STATE $(\D_t, \D_v) \leftarrow$ Partition $\D$ into $\D_t$ for model training and $\D_v$ for calibrating the conformal sets;
            \vspace{0.1em}
            \STATE $n' \leftarrow |\D_t|$;
            \vspace{0.1em}
            \STATE $(f, \theta) \leftarrow$ Initialize model $f$ and its weights $\theta$;
            \vspace{0.1em}
            \STATE $(w_{ij})_{i \in [n'], j \in [k]} \leftarrow 1 / |s_i|$ if $j \in s_i$, else $0$;
            \vspace{0.5em}
            \FOR{$r = 1, \dots, R$}
            \vspace{0.1em}
            \STATE \emph{$\triangleright$ Update predictions on the validation set}
            \vspace{0.1em}
            \STATE $\esS \leftarrow \{ \max_{y \in s_i} f_y(x_i; \theta) : (x_i, s_i) \in \D_v \}$;
            \vspace{0.5em}
            \STATE \emph{$\triangleright$ Update $f$'s weights $\theta$}
            \vspace{0.1em}
            \STATE $\hat{R}(f; w, \theta) \leftarrow -\frac{1}{n'} \sum_{i = 1}^{n'} \sum_{j = 1}^{k} w_{ij} \log f_j(x_i; \theta)$;
            \STATE Update $\theta$ by back-propagation on $\!-\nabla \hat{R}(f; w, \theta)$;
            \vspace{0.5em}
            \STATE \emph{$\triangleright$ Clean candidate sets $s_i$}
            \vspace{0.1em}
            \IF{$r \geq R_{\operatorname{warmup}}$}
            \vspace{0.1em}
            \STATE $\alpha_r \leftarrow$ Estimate the adapted error level in~\eqref{eq:corrected-conf};
            \vspace{0.1em}
            \FOR{$(x_i, s_i) \in \D_t$}
            \vspace{0.1em}
            \STATE $C(x_i) \leftarrow$ Construct the conformal predictor as defined in~\eqref{eq:conf-pred} using $\esS$ and $\alpha_r$;
            \vspace{0.1em}
            \IF{$s_i \cap C(x_i) \neq \emptyset$}
            \vspace{0.1em}
            \STATE $s_i \leftarrow s_i \cap C(x_i)$;
            \vspace{0.1em}
            \ENDIF
            \ENDFOR
            \ENDIF
            \vspace{0.5em}
            \STATE \emph{$\triangleright$ Update label weights $w_{ij}$}
            \vspace{0.1em}
            \STATE $(w_{ij})_{i \in [n'], j \in [k]} \leftarrow \frac{ f_j(x_i) }{ \sum_{j' \in s_i} f_{j'}(x_i) }$ if $j \in s_i$, else $0$;
            \vspace{0.1em}
            \ENDFOR
            \vspace{0.1em}
            \STATE {\bfseries return} predictor $f(\,\cdot\,; \theta)$;
        \end{algorithmic}
    \end{algorithm}
    \vspace{-1em}
\end{figure}

\begin{table*}[t!]
    \caption{
        Average test-set accuracies ($\pm$ std.) on the real-world datasets (top) and the supervised datasets with added incorrect candidate labels (bottom).
        We benchmark our strategy (\textsc{Conf+}) combined with all existing methods.
    }
    \label{tab:acc}
    \begin{center}
        \begin{small}
            \begin{tabular}{
                >{\raggedright\arraybackslash}m{0.145\linewidth}
                >{\centering\arraybackslash}m{0.112\linewidth}
                >{\centering\arraybackslash}m{0.112\linewidth}
                >{\centering\arraybackslash}m{0.112\linewidth}
                >{\centering\arraybackslash}m{0.112\linewidth}
                >{\centering\arraybackslash}m{0.112\linewidth}
                >{\centering\arraybackslash}m{0.112\linewidth}
                }
                \toprule
                Method                                 & \emph{bird-song}                   & \emph{lost}                        & \emph{mir-flickr}                    & \emph{msrc-v2}                     & \emph{soccer}                      & \emph{yahoo-news}                  \\
                \midrule
                \mbox{\textsc{Proden} (2020)}          & \mbox{\,75.55 ($\pm$ 1.08)}        & \mbox{\,78.94 ($\pm$ 3.01)}        & \mbox{\textbf{67.05} ($\pm$ 1.18)}   & \mbox{\,54.33 ($\pm$ 1.76)}        & \mbox{\,54.18 ($\pm$ 0.55)}        & \mbox{\,65.25 ($\pm$ 1.00)}        \\
                \mbox{\textsc{Conf+P.}\ (no)}          & \mbox{\,76.27 ($\pm$ 0.94)}        & \mbox{\,79.56 ($\pm$ 1.96)}        & \mbox{\,66.07 ($\pm$ 1.63)}          & \mbox{\,53.00 ($\pm$ 2.24)}        & \mbox{\,54.63 ($\pm$ 0.81)}        & \mbox{\,65.42 ($\pm$ 0.36)}        \\
                \mbox{\textbf{\textsc{Conf+Proden}}}   & \mbox{\textbf{76.99} ($\pm$ 0.90)} & \mbox{\textbf{80.09} ($\pm$ 4.40)} & \mbox{\,66.91 ($\pm$ 1.57)}          & \mbox{\textbf{54.60} ($\pm$ 3.42)} & \mbox{\textbf{54.77} ($\pm$ 0.84)} & \mbox{\textbf{65.93} ($\pm$ 0.42)} \\
                \midrule
                \mbox{\textsc{Cc} (2020)}              & \mbox{\,74.49 ($\pm$ 1.57)}        & \mbox{\,78.23 ($\pm$ 2.11)}        & \mbox{\,62.39 ($\pm$ 1.87)}          & \mbox{\,50.96 ($\pm$ 2.03)}        & \mbox{\,55.28 ($\pm$ 0.96)}        & \mbox{\textbf{65.03} ($\pm$ 0.51)} \\
                \mbox{\textbf{\textsc{Conf+Cc}}}       & \mbox{\textbf{75.01} ($\pm$ 1.84)} & \mbox{\textbf{79.38} ($\pm$ 1.79)} & \mbox{\textbf{63.37} ($\pm$ 0.45)}   & \mbox{\textbf{52.45} ($\pm$ 3.64)} & \mbox{\textbf{55.52} ($\pm$ 0.74)} & \mbox{\,64.35 ($\pm$ 0.64)}        \\
                \midrule
                \mbox{\textbf{\textsc{Valen}} (2021)}  & \mbox{\textbf{72.30} ($\pm$ 1.83)} & \mbox{\textbf{70.18} ($\pm$ 3.44)} & \mbox{\textbf{67.05} ($\pm$ 1.48)}   & \mbox{\textbf{49.20} ($\pm$ 1.37)} & \mbox{\textbf{53.20} ($\pm$ 0.88)} & \mbox{\textbf{62.25} ($\pm$ 0.45)} \\
                \mbox{\textsc{Conf+Valen}}             & \mbox{\,71.22 ($\pm$ 1.03)}        & \mbox{\,68.41 ($\pm$ 2.95)}        & \mbox{\,61.61 ($\pm$ 2.79)}          & \mbox{\,48.37 ($\pm$ 2.24)}        & \mbox{\,52.49 ($\pm$ 1.00)}        & \mbox{\,62.16 ($\pm$ 0.74)}        \\
                \midrule
                \mbox{\textbf{\textsc{Cavl}} (2022)}   & \mbox{\,69.78 ($\pm$ 3.00)}        & \mbox{\textbf{72.12} ($\pm$ 1.08)} & \mbox{\textbf{65.02} ($\pm$ 1.34)}   & \mbox{\textbf{52.67} ($\pm$ 2.32)} & \mbox{\textbf{55.06} ($\pm$ 0.48)} & \mbox{\,61.91 ($\pm$ 0.46)}        \\
                \mbox{\textsc{Conf+Cavl}}              & \mbox{\textbf{72.00} ($\pm$ 1.22)} & \mbox{\,71.24 ($\pm$ 3.81)}        & \mbox{\,64.42 ($\pm$ 0.89)}          & \mbox{\,51.63 ($\pm$ 5.03)}        & \mbox{\,54.85 ($\pm$ 0.92)}        & \mbox{\textbf{62.43} ($\pm$ 0.43)} \\
                \midrule
                \mbox{\textsc{Pop} (2023)}             & \mbox{\,75.17 ($\pm$ 1.04)}        & \mbox{\,77.79 ($\pm$ 2.11)}        & \mbox{\,\textbf{67.93} ($\pm$ 1.44)} & \mbox{\,53.83 ($\pm$ 0.69)}        & \mbox{\,55.31 ($\pm$ 0.71)}        & \mbox{\,65.09 ($\pm$ 0.64)}        \\
                \mbox{\textbf{\textsc{Conf+Pop}}}      & \mbox{\textbf{77.58} ($\pm$ 1.01)} & \mbox{\textbf{78.41} ($\pm$ 2.13)} & \mbox{66.21 ($\pm$ 2.19)}            & \mbox{\textbf{54.82} ($\pm$ 3.60)} & \mbox{\textbf{56.49} ($\pm$ 1.10)} & \mbox{\textbf{65.25} ($\pm$ 0.23)} \\
                \midrule
                \mbox{\textsc{CroSel} (2024)}          & \mbox{\,75.11 ($\pm$ 1.79)}        & \mbox{\textbf{81.24} ($\pm$ 3.68)} & \mbox{\textbf{67.58} ($\pm$ 1.16)}   & \mbox{\,52.23 ($\pm$ 2.83)}        & \mbox{\,52.64 ($\pm$ 1.21)}        & \mbox{\textbf{67.72} ($\pm$ 0.32)} \\
                \mbox{\textbf{\textsc{Conf+CroSel}}}   & \mbox{\textbf{77.76} ($\pm$ 0.50)} & \mbox{\,81.15 ($\pm$ 2.57)}        & \mbox{\,65.93 ($\pm$ 1.94)}          & \mbox{\textbf{54.10} ($\pm$ 2.75)} & \mbox{\textbf{54.97} ($\pm$ 0.65)} & \mbox{\,67.55 ($\pm$ 0.22)}        \\
                \midrule
                \midrule
                Method                                 & \emph{mnist}                       & \emph{fmnist}                      & \emph{kmnist}                        & \emph{svhn}                        & \emph{cifar10}                     & \emph{cifar100}                    \\
                \midrule
                \mbox{\textsc{Proden} (2020)}          & \mbox{\,87.21 ($\pm$ 0.83)}        & \mbox{\,71.18 ($\pm$ 2.95)}        & \mbox{\,59.31 ($\pm$ 1.22)}          & \mbox{\,83.71 ($\pm$ 0.37)}        & \mbox{\textbf{86.42} ($\pm$ 0.39)} & \mbox{\textbf{61.58} ($\pm$ 0.20)} \\
                \mbox{\textbf{\textsc{Conf+P.}\ (no)}} & \mbox{\textbf{91.74} ($\pm$ 0.34)} & \mbox{\textbf{78.38} ($\pm$ 0.50)} & \mbox{\textbf{66.88} ($\pm$ 0.76)}   & \mbox{\textbf{87.31} ($\pm$ 0.30)} & \mbox{\,85.39 ($\pm$ 0.49)}        & \mbox{\,61.50 ($\pm$ 0.20)}        \\
                \mbox{\textsc{Conf+Proden}}            & \mbox{\,91.55 ($\pm$ 0.23)}        & \mbox{\,78.09 ($\pm$ 0.33)}        & \mbox{\,66.43 ($\pm$ 0.38)}          & \mbox{\,86.99 ($\pm$ 0.41)}        & \mbox{\,85.29 ($\pm$ 0.44)}        & \mbox{\,61.45 ($\pm$ 0.49)}        \\
                \midrule
                \mbox{\textbf{\textsc{Cc}} (2020)}     & \mbox{\textbf{86.29} ($\pm$ 2.18)} & \mbox{\textbf{66.19} ($\pm$ 2.77)} & \mbox{\textbf{58.29} ($\pm$ 0.32)}   & \mbox{\,83.40 ($\pm$ 0.42)}        & \mbox{\textbf{85.61} ($\pm$ 0.27)} & \mbox{\,60.43 ($\pm$ 0.53)}        \\
                \mbox{\textsc{Conf+Cc}}                & \mbox{\,85.20 ($\pm$ 4.16)}        & \mbox{\,59.75 ($\pm$ 2.68)}        & \mbox{\,57.07 ($\pm$ 0.66)}          & \mbox{\textbf{84.32} ($\pm$ 0.31)} & \mbox{\,84.10 ($\pm$ 0.38)}        & \mbox{\textbf{60.49} ($\pm$ 0.37)} \\
                \midrule
                \mbox{\textsc{Valen} (2021)}           & \mbox{\textbf{78.91} ($\pm$ 0.80)} & \mbox{\,66.53 ($\pm$ 2.65)}        & \mbox{\,58.48 ($\pm$ 0.45)}          & \mbox{\,54.87 ($\pm$ 15.83)}       & \mbox{\textbf{84.83} ($\pm$ 0.23)} & \mbox{\,58.67 ($\pm$ 0.17)}        \\
                \mbox{\textbf{\textsc{Conf+Valen}}}    & \mbox{\,74.20 ($\pm$ 21.99)}       & \mbox{\textbf{69.09} ($\pm$ 2.71)} & \mbox{\textbf{60.95} ($\pm$ 2.59)}   & \mbox{\textbf{78.31} ($\pm$ 3.15)} & \mbox{\,84.35 ($\pm$ 0.22)}        & \mbox{\textbf{59.57} ($\pm$ 0.71)} \\
                \midrule
                \mbox{\textbf{\textsc{Cavl}} (2022)}   & \mbox{\,71.11 ($\pm$ 3.92)}        & \mbox{\textbf{59.85} ($\pm$ 6.49)} & \mbox{\,48.15 ($\pm$ 5.07)}          & \mbox{\textbf{72.57} ($\pm$ 3.14)} & \mbox{\textbf{84.00} ($\pm$ 0.94)} & \mbox{\textbf{61.97} ($\pm$ 0.25)} \\
                \mbox{\textsc{Conf+Cavl}}              & \mbox{\textbf{71.86} ($\pm$ 4.57)} & \mbox{\,59.54 ($\pm$ 6.62)}        & \mbox{\textbf{52.14} ($\pm$ 3.89)}   & \mbox{\,70.53 ($\pm$ 2.94)}        & \mbox{\,82.82 ($\pm$ 1.58)}        & \mbox{\,61.79 ($\pm$ 0.36)}        \\
                \midrule
                \mbox{\textsc{Pop} (2023)}             & \mbox{\,87.08 ($\pm$ 0.58)}        & \mbox{\,72.30 ($\pm$ 2.63)}        & \mbox{\,60.63 ($\pm$ 1.15)}          & \mbox{\,83.69 ($\pm$ 0.28)}        & \mbox{\textbf{86.76} ($\pm$ 0.29)} & \mbox{\,61.27 ($\pm$ 0.60)}        \\
                \mbox{\textbf{\textsc{Conf+Pop}}}      & \mbox{\textbf{91.19} ($\pm$ 0.29)} & \mbox{\textbf{79.15} ($\pm$ 1.23)} & \mbox{\textbf{67.37} ($\pm$ 0.28)}   & \mbox{\textbf{85.89} ($\pm$ 0.48)} & \mbox{\,85.32 ($\pm$ 0.38)}        & \mbox{\textbf{61.38} ($\pm$ 0.30)} \\
                \midrule
                \mbox{\textsc{CroSel} (2024)}          & \mbox{\,91.84 ($\pm$ 0.44)}        & \mbox{\,76.34 ($\pm$ 1.21)}        & \mbox{\textbf{65.55} ($\pm$ 0.81)}   & \mbox{\,75.95 ($\pm$ 3.91)}        & \mbox{\textbf{87.32} ($\pm$ 0.22)} & \mbox{\,63.69 ($\pm$ 0.29)}        \\
                \mbox{\textbf{\textsc{Conf+CroSel}}}   & \mbox{\textbf{91.85} ($\pm$ 0.61)} & \mbox{\textbf{77.31} ($\pm$ 0.46)} & \mbox{\,64.73 ($\pm$ 1.52)}          & \mbox{\textbf{77.70} ($\pm$ 3.84)} & \mbox{\,87.05 ($\pm$ 0.09)}        & \mbox{\textbf{64.55} ($\pm$ 0.31)} \\
                \bottomrule
            \end{tabular}
        \end{small}
    \end{center}
\end{table*}

In \textbf{phase~1} (Lines~6--7), we use the current model $f$ to predict the labels on the hold-out validation set $\D_v$, which are required for the computations in phase~3.

In \textbf{phase~2} (Lines~8--10), we update the weights $\theta$ of the neural network $f$ by performing back-propagation on the risk term~\eqref{eq:erm}.
As our candidate cleaning method is agnostic to the concrete PLL classifier used, one can also use other commonly-used PLL strategies instead.

In \textbf{phase~3} (Lines~11--18), we compute the conformal predictor $C$, which is used to clean the candidate sets.
After completing $R_{\operatorname{warmup}}$ warm-up epochs, we start with our pruning procedure.
In Line~13, we compute $\alpha_r$ for the current epoch $r$.
While it is desirable to use the exact value of $\alpha''$ in~\eqref{eq:corrected-conf} in Line~13 of Algorithm~\ref{alg:method}, its computation is unfortunately infeasible as the constants $B$ and $\beta$, for which the Bernstein condition (Assumption~\ref{def:bernstein}) holds, cannot be known unless the true distribution $\prob_{XY}$ is known.
As the employed PLL classifiers are consistent, that is, they converge to the Bayes classifier with enough samples, we approximate the estimation error terms in~\eqref{eq:corrected-conf} by the probability mass that the PLL classifier allocates on false class labels, that is, class labels that are not part of the candidate sets and hence cannot be the correct label.
Given $(x_i, s_i) \in \D_t$, we set $\alpha_r = \frac{1}{n'} \sum_{i = 1}^{n'} \sum_{j \notin s_i} f_j(x_i)$ with $n' = |\D_t|$.
Then, we compute the conformal prediction sets $C(x_i)$ in Line~15 for all training instances $(x_i, s_i) \in \D_t$ using the empirical distribution function of the adapted scores on the validation set; this conformal predictor $C$ is valid by Theorem~\ref{thm:main}.
We use the conformal sets $C(x_i)$ to prune the candidate sets $s_i$.
If $C(x_i)$ and $s_i$ have a nonempty intersection, which is implied with high probability by the conformal validity (Theorem~\ref{prop:pll-valid}), we assign $s_i \cap C(x_i)$ to $s_i$ in Line~17.

Finally, in \textbf{phase~4} (Lines~21--22), we update the label weights $w_{ij}$ based on the cleaned candidate sets $s_i$ with~\eqref{eq:label-update}.

\subsection{Runtime Complexity}
\label{sec:runtime}
The main runtime cost of our cleaning method arises from the computation of the conformal sets $C(x_i)$ in Line~15 of Algorithm~\ref{alg:method}.
Finding the rank of $f_y(x_i)$ within $\esS$ can be done by first sorting $\esS$ and then using a binary search.
This requires a total runtime of $\mathcal{O}(R n \log n)$, as we prune candidate labels in each epoch and, both, the training set $\D_t$ and validation set $\D_v$ have a size of $\mathcal{O}(n)$.
Note that the runtime of our method is not dependent on the number of feature dimensions $d$ as the considered scores $\esS$ are scalars.

\subsection{Placement regarding Related Work}
\label{sec:connection-pop}
In this section, we provide a brief comparison of our cleaning strategy with the one employed by \textsc{Pop} \citep{0009LLQG23}.
\textsc{Pop} uses level sets, which we sketch in the following.
Let $e > 0$, $(x_i, s_i) \in \D$, the predicted label $\hat{y}_{x_i} = \arg\max_{j \in s_i} f_j(x_i)$, and the second-most likely label $\hat{o}_{x_i} = \arg\max_{j \in s_i, j \neq \hat{y}} f_j(x_i)$.
The level sets are of the form $L(e) = \{ x \in \X : f_{\hat{y}_{x}}(x) - f_{\hat{o}_{x}}(x) \geq e \}$ to gradually clean the candidate labels for instances in $L(e)$.
In other words, one is confident in the predicted labels if the distance between the most likely and second-most likely label exceeds some margin.
Given $x \in L(e)$, this implies
\begin{align}
                                               & f_{\hat{y}_{x}}(x) - f_{\hat{o}_{x}}(x) \geq e                                                                                \\
    \overset{(\dagger)}{\Leftrightarrow} \quad & 2 f_{\hat{y}_{x}}(x) - 1 + \underbrace{ \sum_{j' \in \Y \setminus \{ \hat{y}_{x}, \hat{o}_{x} \}} f_{j'}(x) }_{\leq 1} \geq e \\[0.1em]
    \Rightarrow \quad                          & f_{\hat{y}_{x}}(x) \geq \frac{1}{2} e \text{,}
    \label{eq:pop}
\end{align}
with $(\dagger)$ holding as
$f_{\hat{o}_{x}}(x) = 1 - \sum_{j' \in \Y \setminus \{ \hat{y}_{x}, \hat{o}_{x} \}} f_{j'}(x) - f_{\hat{y}_{x}}(x)$.
\textsc{Pop} gradually decreases the value of $e$ to enlarge the reliable region $L(e)$, which in turn requires $f_{\hat{y}_{x}}(x) \geq \frac{1}{2} e$ by~\eqref{eq:pop}.
In contrast, in Theorem~\ref{thm:main}, we find an appropriate value $t$ such that $f_{\hat{y}_{x}}(x) \geq t$ holds with a specified probability.
The conformal predictor $C$ can therefore also be interpreted as a level set.
However, our approach satisfies the conformal validity guarantee.

\section{Experiments}
\label{sec:exp}
Section~\ref{sec:methods} lists all PLL methods that we compare against,
Section~\ref{sec:setup} summarizes the experimental setup,
and Section~\ref{sec:results} shows our results.

\subsection{Algorithms for Comparison}
\label{sec:methods}
In our experiments, we benchmark six state-of-the-art PLL methods.
These are \textsc{Proden} \citep{LvXF0GS20}, \textsc{Cc} \citep{FengL0X0G0S20}, \textsc{Valen} \citep{XuQGZ21}, \textsc{Cavl} \citep{ZhangF0L0QS22}, \textsc{Pop} \citep{0009LLQG23}, and \textsc{CroSel} \citep{crosel2024}.
For each dataset, we use the same base models across all approaches.
For the colored-image datasets, we use a ResNet-9 architecture \citep{HeZRS16}.
Else, we use a standard $d$-300-300-300-$k$ MLP \citep{werbos1974beyond}.
We train all models from scratch.
An in-depth overview of all hyperparameters is in Appendix~\ref{sec:hypar}.
Appendix~\ref{sec:more-exp} contains additional experiments, including the use of the pre-trained \textsc{Blip-2} model \citep{0008LSH23} on the vision datasets.

\subsection{Experimental Setup}
\label{sec:setup}

\paragraph{Data.}
Using the standard PLL experimentation protocol \citep{LvXF0GS20,ZhangF0L0QS22,crosel2024}, we perform experiments on real-world PLL datasets and on supervised datasets with artificially added incorrect candidate labels.
To report averages and standard deviations, we repeat all experiments five times with different seeds.
For the supervised multi-class datasets, we use \emph{mnist} \citep{uci-mnist}, \emph{fmnist} \citep{uci-fmnist}, \emph{kmnist} \citep{uci-kmnist}, \emph{cifar10} \citep{Krizhevsky2009LearningML}, \emph{cifar100} \citep{Krizhevsky2009LearningML}, and \emph{svhn} \citep{NetzerSVHN11}.
Regarding the real-world PLL datasets, we use \emph{bird-song} \citep{BriggsFR12}, \emph{lost} \citep{CourST11}, \emph{mir-flickr} \citep{HuiskesL08}, \emph{msrc-v2} \citep{LiuD12}, \emph{soccer} \citep{ZengXJCGXM13}, and \emph{yahoo-news} \citep{GuillauminVS10}.
An overview of the dataset characteristics is in~Appendix~\ref{sec:hypar}.

\paragraph{Candidate generation.}
As is common in related work, we use two kinds of candidate label generation methods to augment labeled multi-class data with partial labels: Uniform \citep{HullermeierB06,LiuD12} and instance-dependent \citep{XuQGZ21}.
For \emph{cifar10} and \emph{cifar100}, we use the uniform generation strategy as in \citet{WangXLF0CZ22} and the instance-dependent strategy for all other datasets.
For adding instance-dependent candidate labels, we first train a supervised MLP classifier $g : \X \to [0, 1]^k$.
Then, given an instance $x \in \X$ with correct label $y \in \Y$, we add the incorrect label $\bar{y} \in \Y \setminus \{y\}$ to the candidate set $s$ with a binomial flipping probability of
$\xi_{\bar{y}}(x) = g_{\bar{y}}(x) / \max_{y' \in \Y \setminus \{y\}} g_{y'}(x)$.
For \emph{cifar10}, we use a constant flipping probability of $\xi_{\bar{y}}(x) = 0.1$.
In the \emph{cifar100} dataset, all class labels $\Y$ are partitioned into 20 meta-categories (for example, aquatic mammals consisting of the labels beaver, dolphin, otter, seal, and whale) and we use a constant flipping probability of $\xi_{\bar{y}}(x) = 0.1$ if $\bar{y}$ and $y$ belong to the same meta-category, else we set $\xi_{\bar{y}}(x) = 0$.

\subsection{Results}
\label{sec:results}

\begin{table}[t!]
    \caption{
        Number of significant differences compared to all 6 methods on all 12 datasets using a paired t-test (level \SI{5}{\%}).
    }
    \label{tab:comp}
    \begin{center}
        \begin{small}
            \begin{tabular}{
                >{\raggedright\arraybackslash}m{0.54\linewidth}
                >{\centering\arraybackslash}m{0.07\linewidth}
                >{\centering\arraybackslash}m{0.08\linewidth}
                >{\centering\arraybackslash}m{0.09\linewidth}
                }
                \toprule
                Comparison vs. all others            & Wins         & Ties        & Losses       \\
                \midrule
                \textsc{Proden} (2020)               & 26           & \textbf{36} & 10           \\
                \textsc{Conf+Proden} (no correction) & \textbf{37}  & 24          & 11           \\
                \textsc{Conf+Proden}                 & \textbf{44}  & 21          & \phantom{0}7 \\
                \midrule
                \textsc{Cc} (2020)                   & 17           & \textbf{36} & 19           \\
                \textsc{Conf+Cc}                     & 19           & \textbf{28} & 25           \\
                \midrule
                \textsc{Valen} (2021)                & \phantom{0}3 & 31          & \textbf{38}  \\
                \textsc{Conf+Valen}                  & \phantom{0}4 & 26          & \textbf{42}  \\
                \midrule
                \textsc{Cavl} (2022)                 & \phantom{0}8 & 28          & \textbf{36}  \\
                \textsc{Conf+Cavl}                   & \phantom{0}5 & 29          & \textbf{38}  \\
                \midrule
                \textsc{Pop} (2023)                  & 27           & \textbf{38} & \phantom{0}7 \\
                \textsc{Conf+Pop}                    & \textbf{44}  & 22          & \phantom{0}6 \\
                \midrule
                \textsc{CroSel} (2024)               & \textbf{36}  & 29          & \phantom{0}7 \\
                \textsc{Conf+CroSel}                 & \textbf{49}  & 19          & \phantom{0}4 \\
                \bottomrule
            \end{tabular}
        \end{small}
    \end{center}
\end{table}

\paragraph{Predictive performance.}
Table~\ref{tab:acc} presents the average test-set accuracies for all competitors on all datasets.
We benchmark our conformal candidate cleaning technique combined with all approaches in Section~\ref{sec:methods}, which is marked by \textsc{Conf+Method}.
An overview of significant differences is in Table~\ref{tab:comp}.
There, we compare the respective method to all the other approaches.
All significance tests use a paired student t-test with a confidence level of \SI{5}{\%}.

The approaches \textsc{Conf+Proden}, \textsc{Conf+Pop}, and \textsc{Conf+CroSel} that combine the respective approaches with our candidate cleaning strategy win most often (Table~\ref{tab:comp}).
Conformal candidate cleaning makes \textsc{Proden} win 18 more direct comparisons, \textsc{Pop} win 17 more direct comparisons, and \textsc{CroSel} win 13 more direct comparisons advancing the state-of-the-art prediction performance.
These methods significantly benefit from our pruning.

The approaches \textsc{Cc}, \textsc{Valen}, and \textsc{Cavl} yield similar performances when combined with conformal candidate cleaning.
For \textsc{Valen} and \textsc{Cavl}, we attribute this to the fact that their methods already use pseudo-labeling internally, that is, they treat the most likely label as the possible correct label, which diminishes the positive effect of pruning candidates.

\paragraph{Ablation study.}
Additionally, we perform an ablation experiment regarding our correction method proposed in Theorem~\ref{thm:main}.
The approach \textsc{Conf+Proden} (no correction) uses conformal predictions based on the labels provided by the PLL classifier without our proposed correction method, which is equivalent to a fixed $\alpha_r$.
Table~\ref{tab:comp} shows that, while \textsc{Conf+Proden} (no correction) is already a significant improvement over \textsc{Proden}, our PLL correction strategy improves performance even further by incorporating the possible approximation error of the trained classifier.
We limit our ablation study to \textsc{Proden} due to runtime constraints.

Our experiments show that the proposed method yields significant improvements over a wide range of existing PLL models (\textsc{Proden}, \textsc{Pop}, and \textsc{CroSel}) and advances the state-of-the-art prediction performance with the method \textsc{Conf+CroSel}.

%
%

\begin{acknowledgements}
    This work was supported by the German Research Foundation (DFG) Research Training Group GRK 2153: \emph{Energy Status Data --- Informatics Methods for its Collection, Analysis and Exploitation} and by the pilot program Core-Informatics of the Helmholtz Association (HGF).
\end{acknowledgements}

%
%

\bibliography{references}

%
%

\newpage
\onecolumn

\title{Partial-Label Learning with Conformal Candidate Cleaning\\(Supplementary Material)}
\maketitle

\appendix
\section{Proofs}
\label{sec:proofs}
This section collects our proofs.
Section~\ref{sec:proof1} contains the proof of Theorem~\ref{prop:pll-valid}, Section~\ref{sec:proof2} that of Lemma~\ref{lemma:main}, and Section~\ref{sec:proof3} that of Theorem~\ref{thm:main}.

\subsection{Proof of Theorem~\ref{prop:pll-valid}}
\label{sec:proof1}
Let $C$ be an optimal solution of~\eqref{eq:opt-supervised}.
Then, we have
\begin{align}
    \prob_{XS}\!\left( S \cap C(X) \neq \emptyset \right)
     & = 1 - \prob_{XS}\!\left( S \cap C(X) = \emptyset \right)
    = 1 - \prob_{XS}\!\left( \forall y \in S, y \notin C(X) \right)
    \geq 1 - \prob_{XS}\!\left( \exists y \in S, y \notin C(X) \right)   \notag                                                                    \\
     & \overset{(a)}{=} 1 - \sum_{y \in \Y} \prob\!\left( Y = y, y \in S, y \notin C(X) \right)
    = 1 - \prob\!\left( Y \in S, Y \notin C(X) \right)                                                                                             \\
     & \overset{(b)}{=} 1 - \int_{\X \times \Y} \prob_{S \mid X = x, Y = y}\!\left(y \in S, y \notin C(x) \right) \, \mathrm{d} \prob_{XY}( x, y ) \\
     & \overset{(c)}{=} 1 - \int_{\X \times \Y}
    \underbrace{ \prob_{S \mid X = x, Y = y}\!\left( y \in S \right) }_{\overset{(d)}{=} 1}
    \prob_{S \mid X = x, Y = y, y \in S}\!\left( y \notin C(x) \right) \,
    \mathrm{d} \prob_{XY}( x, y )                                                                                                                  \\
     & = 1 - \int_{\X \times \Y}
    \prob_{S \mid X = x, Y = y, y \in S}\!\left( y \notin C(x) \right) \,
    \mathrm{d} \prob_{XY}( x, y )                                                                                                                  \\
     & \overset{(e)}{=} 1 - \int_{\X \times \Y} \prob_{S \mid X = x, Y = y}\!\left( y \notin C(x) \right) \, \mathrm{d} \prob_{XY}( x, y )         \\
     & \overset{(f)}{=} 1 - \int_{\X \times \Y} \mathds{1}_{\{ y \notin C(x) \}} \, \mathrm{d} \prob_{XY}( x, y )                                  \\
     & = 1 - \prob_{XY}\!\left( Y \notin C(X) \right)
    = \prob_{XY}\!\left( Y \in C(X) \right)
    \overset{(g)}{\geq} 1 - \alpha \text{,}
\end{align}
where
$(a)$ is implied by the law of total probability holding for the discrete $Y$ taking mutually exclusive values in $y \in \Y$,
$(b)$ holds by the tower rule,
$(c)$ holds by the chain rule of conditional probability,
$(d)$ holds as $\prob_{S \mid X = x, Y = y}(y \in S) = 1$ for any $(x, y) \in \X \times \Y$,
$(e)$ holds by independence,
$(f)$ holds as $\prob_{S \mid X = x, Y = y}(y \notin C(x))$ is either one if $y \notin C(x)$ or zero if $y \in C(x)$, and
$(g)$ holds by our imposed assumption.

\subsection{Proof of Lemma~\ref{lemma:main}}
\label{sec:proof2}

We prove parts (i) and (ii) separately in the following.

\paragraph{Proof of $(i)$.} To proof the result, we first show that for any $\hat f$, one has the expectation bound
\begin{align}
    \expected_{XY}\!\left[ \left| \hat{f}_Y(X) - f^*_Y(X) \right| \right]
    \leq \frac{\lambda \sqrt{B}}{2^{\frac{1}{2}\beta}} \left( R(\hat{f}) - R(f^*) \right)^{\frac{1}{2}\beta},\label{eq:lemma-first-step}
\end{align}
for some constants $\beta \in (0, 1]$ and $B, \lambda > 0$. We then apply a known result (recalled in Theorem~\ref{thm:estimation-error}) to obtain the stated concentration inequality.
The details are as follows.

To prove \eqref{eq:lemma-first-step}, notice that
\begin{align}
    \MoveEqLeft\expected_{XY}\!\left[ \left| \hat{f}_Y(X) - f^*_Y(X) \right| \right]
    = \left( \left( \expected_{XY}\!\left[ \left| \hat{f}_Y(X) - f^*_Y(X) \right| \right] \right)^2 \right)^{\frac{1}{2}}
    \overset{(a)}{\leq} \left( \expected_{XY}\!\left[ \left| \hat{f}_Y(X) - f^*_Y(X) \right|^2 \right] \right)^{\frac{1}{2}}                                                \\
     & \overset{(b)}{\leq} \lambda \left( \expected_{XY}\!\left[ \left| \ell(\hat{f}(X), Y) - \ell(f^*(X), Y) \right|^2 \right] \right)^{\frac{1}{2}}
    = \lambda \left( \expected_{XY}\!\left[ \left( \ell(\hat{f}(X), Y) - \ell(f^*(X), Y) \right)^2 \right] \right)^{\frac{1}{2}}                                            \\
     & \overset{(c)}{\leq} \lambda \sqrt{B} \left( \expected_{XY}\!\left[ \ell(\hat{f}(X), Y) - \ell(f^*(X), Y) \right] \right)^{\frac{1}{2}\beta}
    \overset{(d)}{=} \lambda \sqrt{B} \left( \expected_{XY}\!\left[ \ell(\hat{f}(X), Y) \right] - \expected_{XY}\!\left[ \ell(f^*(X), Y) \right] \right)^{\frac{1}{2}\beta} \\
     & \overset{(e)}{=} \lambda \sqrt{B} \frac{1}{2^{\frac{1}{2}\beta}} \left( R(\hat{f}) - R(f^*) \right)^{\frac{1}{2}\beta} \text{,}
\end{align}
using the following observations.
(a) is implied by Jensen's inequality.
Next, we note that $z \mapsto -\log(z)$ satisfies the $\lambda$-bi-Lipschitz condition on $[\epsilon,1]$ (Lemma~\ref{lemma:lipschitz}), implying that
\begin{align}
    \left|\hat f_Y(X) - f^*_Y(X)\right| \le \lambda \left|-\log\left(\hat f_Y(X)\right) -\left(-\log\left( f_Y^*(X)\right)\right)\right| = \left|\ell\left( \hat f(X), Y\right)-\ell\left(f^*(X), Y\right)\right|, \label{eq:lipschitz-inequality}
\end{align}
where we used the definition of $\ell$ for the equality. Using \eqref{eq:lipschitz-inequality} together with the monotonicity of the $L_2$-norm yields (b).
$(c)$ holds by the assumed $(\beta, B)$-Bernstein condition. The linearity of expectations gives (d) and an identity recalled in Theorem~\ref{thm:sup-risk} yields (e).

Now, to obtain the probabilistic bound, we observe that
\begin{align}
    \MoveEqLeft\prob_n\!\Bigg( \expected_{XY}\!\left[ | \hat{f}_Y(X) - f^*_Y(X) | \right]
    \leq M_1 \left( \frac{\log(1 / \delta_1)}{n} \right)^{\frac{1}{4}\beta} \Bigg)                                                                                                                                                                     \\
     & \overset{\eqref{eq:lemma-first-step}}{\ge} \prob_n\!\Bigg( \lambda \sqrt{B} \frac{1}{2^{\frac{1}{2}\beta}} \left( R(\hat{f}) - R(f^*) \right)^{\frac{1}{2}\beta} \leq M_1 \left( \frac{\log(1 / \delta_1)}{n} \right)^{\frac{1}{4}\beta} \Bigg) \\
     & \overset{(a)}{=} \prob_n\!\Bigg( \lambda^{\frac 2\beta} B^{\frac1\beta} \frac{1}{2} \left( R(\hat{f}) - R(f^*) \right) \leq M_1 \left( \frac{\log(1 / \delta_1)}{n} \right)^{\frac{1}{2}} \Bigg)
    \overset{(b)}{\geq} 1 - \delta_1 .
\end{align}
In (a), we notice that both sides of the inequality are nonnegative and apply the function $z \mapsto z^{2/\beta}$, which is monotonically increasing on $\mathbb R^+$. We conclude the proof of part (i) with an application of Theorem~\ref{thm:estimation-error} in (b), where we let $M_1 = M  \lambda^{\frac 2\beta} B^{\frac1\beta} \frac{1}{2}$ (with $M$ defined in the external result).

\paragraph{Proof of $(ii)$.} The proof of the lemma proceeds in three steps. In step 1, we will show that
\begin{align}
    \expected_{XY}\!\left[ \mathds{1}_{\{ \arg\max_{j \in \Y} \hat{f}_j(X) \neq \arg\max_{j \in \Y} f^*_j(X) \}} \right]
    \le  \frac{1}{1 - \delta_5} \expected_{XY}\!\left[ \left(\mathds{1}_{\{ \arg\max_{j \in \Y} \hat{f}_j(x) \neq y \}} - \mathds{1}_{\{ \arg\max_{j \in \Y} f^*_j(x) \neq y \}}\right)^2 \right],
\end{align}
which will allow us to obtain, in step 2, that for some constants $\beta \in (0, 1]$, $B > 0$ and $\delta_5 \in [0, 1)$, one has
\begin{align}
    \prob_X\!\left[ \arg\max_{j \in \Y} \hat{f}_j(X) \neq \arg\max_{j \in \Y} f^*_j(X) \right]
    \leq \frac{B}{1 - \delta_5} \left( R(\hat{f}) - R(f^*) \right)^\beta \text{.} \label{eq:part-ii-probability-bound}
\end{align}
The result will then follow by an application of Theorem~\ref{thm:estimation-error}, which we elaborate in step 3. The details are as follows.

\textbf{Step 1.} Note that we have
\begin{align}
    1 \overset{(a)}{\le} \frac{1}{1-\delta_5} \prob_{Y \mid X = x}\left( Y \in \{ \hat{y}_x, y^*_x \}\right) \overset{(b)}{\le}  \frac{1}{1-\delta_5}\sum_{y\in \{\hat{y}_x, y^*_x\}} \prob_{Y \mid X = x}\left( Y = y \right), \label{eq:part-ii-assumption}
\end{align}
with the assumption used in (a) and a union bound implying (b).

To conclude the first step, we obtain
\begin{align}
    \MoveEqLeft\expected_{XY}\!\left[ \mathds{1}_{\{ \arg\max_{j \in \Y} \hat{f}_j(X) \neq \arg\max_{j \in \Y} f^*_j(X) \}} \right]
    \\
     & \overset{(a)}{\leq} \frac{1}{1 - \delta_5} \expected_{X}\!\left[ \sum_{y \in \{ y^*_X, \hat{y}_X \}} \prob_{Y \mid X}(Y = y) \mathds{1}_{\{ \hat{y}_X \neq y_X^* \}} \right]                                                             \\
     & \overset{(b)}{=} \frac{1}{1 - \delta_5} \expected_{X}\!\left[ \sum_{y \in \{ y_X^*, \hat{y}_X \}} \prob_{Y \mid X}(Y = y) (\mathds{1}_{\{ \hat{y}_X \neq y \}} - \mathds{1}_{\{ y_X^* \neq y \}})^2 \right]                              \\
     & \overset{(c)}{\leq} \frac{1}{1 - \delta_5} \expected_{X}\!\left[ \sum_{y \in \Y} \prob_{Y \mid X}(Y = y) (\mathds{1}_{\{ \hat{y}_X \neq y \}} - \mathds{1}_{\{ y_X^* \neq y \}})^2 \right]                                               \\
     & \overset{(d)}{=} \frac{1}{1 - \delta_5} \expected_{XY}\!\left[ \left(\mathds{1}_{\{ \arg\max_{j \in \Y} \hat{f}_j(x) \neq y \}} - \mathds{1}_{\{ \arg\max_{j \in \Y} f^*_j(x) \neq y \}}\right)^2 \right] \label{eq:lemma-ii-first-step}
    \text{,}
\end{align}
where (a) is implied by \eqref{eq:part-ii-assumption} and the indicator function being nonnegative. For (b), we must show that  $\mathds{1}_{\{ \hat{y}_x \neq y_x^* \}} = (\mathds{1}_{\{ \hat{y}_x \neq y \}} - \mathds{1}_{\{ y_x^* \neq y \}})^2$ for any (fixed) $x \in \X$ and $y \in \{ \hat{y}_x, y_x^* \}$; it suffices to check the three cases.
\begin{itemize}
    \item If $y = \hat{y}_x = y_x^*$, then $0 = 0$,
    \item if $\hat{y}_x \neq y_x^*$ and $y = \hat{y}_x$, then $1 = 1$, and
    \item if $\hat{y}_x = y_x^*$ and $y \neq \hat{y}_x$, then $1 = 1$.
\end{itemize}
In (c), we add nonnegative terms and (d) holds by the definition of the expectation.

\textbf{Step 2.} We relax the l.h.s.\ in \eqref{eq:part-ii-probability-bound} to
\begin{align}
    \MoveEqLeft\prob_X\!\left[ \arg\max_{j \in \Y} \hat{f}_j(X) \neq \arg\max_{j \in \Y} f^*_j(X) \right]
    \overset{(a)}{=} \expected_{X}\!\left[ \mathds{1}_{\{ \arg\max_{j \in \Y} \hat{f}_j(X) \neq \arg\max_{j \in \Y} f^*_j(X) \}} \right]                                                                 \\
     & \overset{(b)}{=} \expected_{XY}\!\left[ \mathds{1}_{\{ \arg\max_{j \in \Y} \hat{f}_j(X) \neq \arg\max_{j \in \Y} f^*_j(X) \}} \right]                                                             \\
     & \overset{(c)}{\leq} \frac{1}{1 - \delta_5} \expected_{XY}\!\left[ (\mathds{1}_{\{ \arg\max_{j \in \Y} \hat{f}_j(X) \neq Y \}} - \mathds{1}_{\{ \arg\max_{j \in \Y} f^*_j(X) \neq Y \}})^2 \right] \\
     & \overset{(d)}{=} \frac{1}{1 - \delta_5} \expected_{XY}\!\left[ (\ell(\hat{f}(X), Y) - \ell(f^*(X), Y))^2 \right]
    \overset{(e)}{\leq} \frac{B}{1 - \delta_5} \left( R(\hat{f}) - R(f^*) \right)^\beta  \text{,}
\end{align}
obtaining the r.h.s.\ and establishing \eqref{eq:part-ii-probability-bound}. The details are as follows. In (a), we use that a probability can be written as the expectation of an indicator function. We notice in (b) that the integrand does not depend on $Y$.
Regarding (c), with $\hat y_x, y_x^*$ defined as in the statement, we use \eqref{eq:lemma-ii-first-step} obtained in step 1. Defining $\ell  : [0, 1]^k \to \mathbb{R}_{\geq 0}$, $(p, y) \mapsto \mathds{1}_{\{ \arg\max_{y' \in \Y} p_{y'} \neq y \}}$ as the usual 0-1-loss yields (d) and  the $(\beta, B)$-Bernstein condition gives (e).

\textbf{Step 3.} It remains to obtain the probabilistic bound. We have that
\begin{align}
    \MoveEqLeft\mathbb P_n\left(\prob_X\!\left[ \arg\max_{j \in \Y} \hat{f}_j(X) \neq \arg\max_{j \in \Y} f^*_j(X) \right]  \leq M_2 \left( \frac{\log(1 / \delta_2)}{n} \right)^{\frac{1}{2}\beta} \right)                                            \\
     & \overset{\eqref{eq:part-ii-probability-bound}}{\ge} \mathbb P_n\left(\frac{B}{1 - \delta_5} \left( R(\hat{f}) - R(f^*) \right)^\beta  \leq M_2 \left( \frac{\log(1 / \delta_2)}{n} \right)^{\frac{1}{2}\beta} \right)                           \\
     & \overset{(a)}{=} \mathbb P_n\left(\left(\frac{B}{1 - \delta_5}\right)^{\frac1\beta} \left( R(\hat{f}) - R(f^*) \right)  \leq M_2^{\frac1\beta} \left( \frac{\log(1 / \delta_2)}{n} \right)^{\frac{1}{2}} \right) \overset{(b)}{\ge} 1-\delta_2,
\end{align}
where we apply the monotonically increasing $z\mapsto z^{1/\beta}$ in (a). In (b), we set $M_2^{\frac1\beta}  = M \left(\frac{B}{1 - \delta_5}\right)^{\frac1\beta}$ and apply Theorem~\ref{thm:estimation-error} (with $M$ given there). This concludes the proof of (ii).

\subsection{Proof of Theorem~\ref{thm:main}}
\label{sec:proof3}
To obtain the statement, we first show that one has the following decomposition.
For any $\alpha \in (0, 1)$ and some $\delta_3 \in (0, 1)$,
\begin{align}
    \prob_{X}\!\left[ f^*_{y^*_{X}}(X) \geq t_\alpha - \delta_3 \right] \ge
    \underbrace{\prob_{X}\!\left[ f^*_{y^*_X}(X) \geq \hat{f}_{y^*_X}(X) - \delta_3 \right]}_{=:t_1}
    + \underbrace{\prob_{X}\!\left[ \hat{f}_{y^*_X}(X) = \hat{f}_{\hat{y}_X}(X) \right]}_{=:t_2}
    + \underbrace{\prob_{X}\!\left[ \hat{f}_{\hat{y}_X}(X) \geq t_\alpha \right]}_{=:t_3}
    - 2. \hspace{0.4cm} \label{eq:incl-excl}
\end{align}
We will then obtain lower bounds on the individual terms $t_1$, $t_2$, and $t_3$, and show that their combination implies the stated result.

\textbf{Decomposition.} Let
$A_1 = \{ X : f^*_{y^*_X}(X) \geq \hat{f}_{y^*_X}(X) - \delta_3 \}$,
$A_2 = \{ X : \hat{f}_{y^*_X}(X) = \hat{f}_{\hat{y}_X}(X) \}$,
$A_3 = \{ X : \hat{f}_{\hat{y}_X}(X) \geq t_\alpha \}$, and
$B = \{ X : f^*_{y^*_{X}}(X) \geq t_\alpha - \delta_3 \}$.
Using these definitions, we obtain that
\begin{align}
    \MoveEqLeft\prob_{X}\!\left[ f^*_{y^*_{X}}(X) \geq t_\alpha - \delta_3 \right]
    \overset{(a)}{=} \mathbb P_X\left[B\right]
    \overset{(b)}{\ge}
    \prob_{X}[ A_1 \cap A_2 \cap A_3 ]
    \overset{(c)}{=} 1 - \prob_{X}[ (A_1 \cap A_2 \cap A_3)^c ]
    \overset{(d)}{=} 1 - \prob_{X}[ A_1^c \cup A_2^c \cup A_3^c ]                                                                                                        \\
     & \overset{(e)}{\geq} 1 - \prob_{X}[ A_1^c ] - \prob_{X}[ A_2^c ] - \prob_{X}[ A_2^c ]
    \overset{(f)}{=} 1 - (1 - \prob_{X}[ A_1 ]) - (1 - \prob_{X}[ A_2 ]) - (1 - \prob_{X}[ A_3 ])                                                                        \\
     & \overset{(g)}{=} \underbrace{\prob_{X}[ A_1 ]}_{=t_1} + \underbrace{\prob_{X}[ A_2 ]}_{=t_2} + \underbrace{\prob_{X}[ A_3 ]}_{=t_3} - 2, \label{eq:decomposition}
\end{align}
with the following details. (a) is by the preceeding definition of the $B$ set. For (b), we have to show that $B \supseteq A_1 \cap A_2 \cap A_3$, which implies that $ \mathbb P_X\left[B\right] \ge \mathbb P_X\left[A_1 \cap A_2 \cap A_3\right]$. Let $x \in A_1 \cap A_2 \cap A_3$, then
\begin{align}
    f^*_{y^*_x}(x) \geq \hat{f}_{y^*_x}(x) - \delta_3 = \hat{f}_{\hat{y}_x}(x) - \delta_3 \geq t_\alpha - \delta_3
    \quad\implies\quad f^*_{y^*_x}(x) \geq t_\alpha - \delta_3 \text{.}
\end{align}
Therefore, $x \in B$, proving (b). (c) considers complementary events and De Morgan's laws yield (d). In (e), we use the inclusion-exclusion principle, where we ignore a few positive terms to obtain the inequality. Considering complementary events gives (f) and cancellations yield (g). This proves \eqref{eq:incl-excl}.

\textbf{Term $t_1$.} To obtain a bound on the first term, we obtain an expectation bound, which together with Markov's inequality and Lemma~\ref{lemma:main}~$(i)$ will give the result. The expectation bound is
\begin{align}
    \MoveEqLeft\expected_{X}\!\left[ | \hat{f}_{y^*_X}(X) - f^*_{y^*_X}(X) | \right]
    \overset{(a)}{\le}  \expected_{X}\!\left[ \frac{\prob_{Y \mid X}(Y = y^*_X) }{1 - \delta_6} \ | \hat{f}_{y^*_X}(X) - f^*_{y^*_X}(X) | \right]
    \\
     & \overset{(b)}{\le} \frac{1}{1 - \delta_6} \expected_{X}\!\Big[ \sum_{y \in \Y} \prob_{Y \mid X}(Y = y) \ | \hat{f}_{y}(X) - f^*_{y}(X) | \Big]
    \overset{(c)}{=} \frac{1}{1 - \delta_6} \expected_{XY}\!\Big[ | \hat{f}_{Y}(X) - f^*_{Y}(X) | \Big] \text{,}
    \label{eq:thm-step-c-2}
\end{align}
with (a) implied by the assumption $\prob_{Y \mid X}(Y = y^*_X) \geq 1 - \delta_6$ guaranteeing that $1\le \frac{\prob_{Y \mid X}(Y = y^*_X) }{1-\delta_6}$. In (b), we use that $y^*_X \in \mathcal Y$ and that all terms in the sum are nonnegative. Using a property of the expectation of a joint distribution yields (c).

Next, Markov's inequality (recalled in Lemma~\ref{lemma:markov}) implies that
\begin{align}
    \MoveEqLeft\prob_{X}\!\left[ | \hat{f}_{y^*_X}(X) - f^*_{y^*_X}(X) | \geq \delta_3 \right]
    \overset{\text{\ref{lemma:markov}}}{\leq} \frac{1}{\delta_3} \expected_{X}\!\left[ | \hat{f}_{y^*_X}(X) - f^*_{y^*_X}(X) | \right]
    \overset{\eqref{eq:thm-step-c-2}}{\leq} \frac{1}{\delta_3 (1 - \delta_6)} \expected_{XY}\!\left[ | \hat{f}_{Y}(X) - f^*_{Y}(X) | \right] \text{.}
    \label{eq:thm-step-c}
\end{align}

Finally, we have that
\begin{align}
    \MoveEqLeft \prob_n\!\left[
        \prob_{X}\!\left[ f^*_{y^*_X}(X) \geq \hat{f}_{y^*_X}(X) - \delta_3 \right]
        \geq 1 - \frac{1}{\delta_3 (1 - \delta_6)} M_1 \left( \frac{\log(1/\delta_1)}{n} \right)^{\frac{1}{4}\beta}
    \right]                                                 \\
     & \overset{(a)}{=} \prob_n\!\left[
    \prob_{X}\!\left[ \hat{f}_{y^*_X}(X) -f^*_{y^*_X}(X)  \le  \delta_3 \right]
    \geq 1 - \frac{1}{\delta_3 (1 - \delta_6)} M_1 \left( \frac{\log(1/\delta_1)}{n} \right)^{\frac{1}{4}\beta}
    \right]                                                 \\
     & \overset{(b)}{\ge} \prob_n\!\left[
    \prob_{X}\!\left[ \left|\hat{f}_{y^*_X}(X) -f^*_{y^*_X}(X)\right|  \le  \delta_3 \right]
    \geq 1 - \frac{1}{\delta_3 (1 - \delta_6)} M_1 \left( \frac{\log(1/\delta_1)}{n} \right)^{\frac{1}{4}\beta}
    \right]                                                 \\
     & \overset{(c)}{\ge} \prob_n\!\left[
    1-\prob_{X}\!\left[ \left|\hat{f}_{y^*_X}(X) -f^*_{y^*_X}(X)\right|  \le  \delta_3 \right]
    \le \frac{1}{\delta_3 (1 - \delta_6)} M_1 \left( \frac{\log(1/\delta_1)}{n} \right)^{\frac{1}{4}\beta}
    \right]                                                 \\
     & \overset{(d)}{\ge} \prob_n\!\left[
    \prob_{X}\!\left[ \left|\hat{f}_{y^*_X}(X) -f^*_{y^*_X}(X)\right|  >  \delta_3 \right]
    \le \frac{1}{\delta_3 (1 - \delta_6)} M_1 \left( \frac{\log(1/\delta_1)}{n} \right)^{\frac{1}{4}\beta}
    \right]                                                 \\
     & \overset{\eqref{eq:thm-step-c}}{\ge} \prob_n\!\left[
        \frac{1}{\delta_3 (1 - \delta_6)} \expected_{XY}\!\left[ | \hat{f}_{Y}(X) - f^*_{Y}(X) | \right]
        \le \frac{1}{\delta_3 (1 - \delta_6)} M_1 \left( \frac{\log(1/\delta_1)}{n} \right)^{\frac{1}{4}\beta}
    \right]                                                 \\
     & \overset{(e)}{\ge} \prob_n\!\left[
        \expected_{XY}\!\left[ | \hat{f}_{Y}(X) - f^*_{Y}(X) | \right]
        \le M_1 \left( \frac{\log(1/\delta_1)}{n} \right)^{\frac{1}{4}\beta}
        \right] \overset{(f)}{\ge} 1-\delta_1 \label{eq:term-t1}
    \text{,}
\end{align}
where we rearrange the l.h.s.\ of the inequality in (a). In (b), we consider the absolute value, decreasing the overall probability. In (c), we subtract $1$ on both sides and multiply by $-1$. In (d), we consider the complement of the event. In (e), we simplify and Lemma~\ref{lemma:main}(i) yields (f).

\textbf{Term $t_2$.} The observation $\prob_{X}\!\left[ \hat{f}_{y^*_X}(X) = \hat{f}_{\hat{y}_X}(X) \right] \geq \prob_{X}\!\left[ y^*_X = \hat{y}_X \right]$ implies that
\begin{align}
    \MoveEqLeft\prob_n\!\left[ \prob_{X}\!\left[ \hat{f}_{y^*_X}(X) = \hat{f}_{\hat{y}_X}(X) \right]
    \geq 1 - M_2 \left( \frac{\log(1/\delta_2)}{n} \right)^{\frac{1}{2}\beta} \right]  \overset{(a)}{\ge}  \prob_n\!\left[\prob_{X}\!\left[ y^*_X = \hat{y}_X \right]
    \ge 1 - M_2 \left( \frac{\log(1/\delta_2)}{n} \right)^{\frac{1}{2}\beta} \right]  \\
     & \overset{(b)}{=} \prob_n\!\left[\prob_{X}\!\left[ y^*_X \neq \hat{y}_X \right]
        \le  M_2 \left( \frac{\log(1/\delta_2)}{n} \right)^{\frac{1}{2}\beta} \right]
    \overset{(c)}{\geq} 1 - \delta_2, \label{eq:term-t2}
\end{align}
where (a) holds by the preceding observation. In (b), we subtract 1 on both sides, multiply by $-1$, and consider the complement of the l.h.s. Inequality (c) was shown in Lemma~\ref{lemma:main}(ii).

\textbf{Term $t_3$.} For bounding the third term, we use the well-known  Dvoretzky-Kiefer-Wolfowitz inequality (recalled in Theorem~\ref{thm:dkw})
In particular, we have
\begin{align}
    \MoveEqLeft\prob_n\!\left[
    \prob_X[ \hat{f}_{\hat{y}_X}(X) \geq t_\alpha ]
    \geq 1 - \left( \alpha + \sqrt{\frac{\log(2/\delta_4)}{2n}} \right)
    \right]
    \overset{(a)}{=} \prob_n\!\left[
    1 - F_{\hat{f}_{\hat{y}_X}(X)}(t_\alpha)
    \geq 1 - \left( \alpha + \sqrt{\frac{\log(2/\delta_4)}{2n}} \right)
    \right]                                                    \\
     & \overset{(b)}{=} \prob_n\!\left[
    F_{\hat{f}_{\hat{y}_X}(X)}(t_\alpha) - \alpha
    \le \sqrt{\frac{\log(2/\delta_4)}{2n}}
    \right]
    \overset{(c)}{\ge} \prob_n\!\left[
    F_{\hat{f}_{\hat{y}_X}(X)}(t_\alpha) - \hat{F}_{\hat{f}_{\hat{y}_X}(X)}(t_\alpha)
    \le \sqrt{\frac{\log(2/\delta_4)}{2n}}
    \right]                                                    \\
     & \overset{(d)}{\ge} \prob_n\!\left[\sup_{t_\alpha}\left|
    F_{\hat{f}_{\hat{y}_X}(X)}(t_\alpha) - \hat{F}_{\hat{f}_{\hat{y}_X}(X)}(t_\alpha)\right|
    \overset{(e)}{\le} \sqrt{\frac{\log(2/\delta_4)}{2n}}
    \right]
    \overset{(e)}{\ge} 1 - \delta_4 \text{,} \label{eq:term-t3}
\end{align}
where (a) holds as
\begin{align}
    \prob_X[ \hat{f}_{\hat{y}_X}(X) \geq t_\alpha ]
    = 1 - \prob_X[ \hat{f}_{\hat{y}_X}(X) \leq t_\alpha ]
    = 1 - F_{\hat{f}_{\hat{y}_X}(X)}(t_\alpha) \text{.}
\end{align}
We rearrange in (b). For obtaining (c), we observe that $\hat{F}_{\hat{f}_{\hat{y}_X}(X)}(t_\alpha) \leq \alpha$. In (d), we consider the supremum, reducing the probability as the inequality becomes more strict. Theorem~\ref{thm:dkw} gives (e).

\textbf{Combination of $t_1$, $t_2$, and $t_3$.}  The desired result is obtained by combining the intermediate results using that
\begin{align}
    \MoveEqLeft \prob_n\!\left[ \prob_X\!\left[ y^*_X \in C(X) \right] \geq 1 - \alpha'_n \right]
    \overset{(\ref{eq:decomposition}a)}{=}  \prob_n\!\left[ \prob_X\!\left[B\right] \geq 1 - \alpha'_n \right]
    \overset{\eqref{eq:decomposition}}{\ge}  \prob_n\!\left[ \prob_X\!\left[A_1\right] + \prob_X\!\left[A_2\right] +  \prob_X\!\left[A_3\right] - 2  \geq 1 - \alpha'_n \right] \\
     & \overset{(a)}{\ge} 1 - (\delta_1 + \delta_2 + \delta_4),
\end{align}
where we use a union bound in (a) and the results obtained in \eqref{eq:term-t1}, \eqref{eq:term-t2}, and \eqref{eq:term-t3}; further, we observe that
\begin{align}
    \MoveEqLeft\left(1 - \frac{1}{\delta_3 (1 - \delta_6)} M_1 \left( \frac{\log(1/\delta_1)}{n} \right)^{\frac{1}{4}\beta} \right)
    + \left(1 - M_2 \left( \frac{\log(1/\delta_2)}{n} \right)^{\frac{1}{2}\beta} \right)
    + \left( 1 - \alpha - \sqrt{\frac{\log(2/\delta_4)}{2n}} \right) - 2                                               \\
     & = 1 - \left( \frac{1}{\delta_3 (1 - \delta_6)} M_1 \left( \frac{\log(1/\delta_1)}{n} \right)^{\frac{1}{4}\beta}
    + M_2 \left( \frac{\log(1/\delta_2)}{n} \right)^{\frac{1}{2}\beta}
    + \alpha
    + \sqrt{\frac{\log(2/\delta_4)}{2n}} \right)
    =   1 - \alpha'_n \text{.}
\end{align}

\section{Auxiliary Results}
\label{sec:aux-results}
This section collects our auxiliary results.

\begin{lemma}
    \label{lemma:lipschitz}
    Let $\varepsilon \in (0, 1)$ and $f : [\varepsilon, 1] \to [0, -\log \varepsilon]$, $z \mapsto - \log z$.
    Then, $f$ is $\frac1\varepsilon$-bi-Lipschitz, that is, for any $x_1, x_2 \in [\varepsilon, 1]$, it holds that
    $\varepsilon | x_1 - x_2 | \leq \left| f(x_1) - f(x_2) \right| \leq \frac1\varepsilon | x_1 - x_2 | $.
\end{lemma}

\begin{proof}
    $f$ is continuous on $[\varepsilon,1]$ and differentiable on $(\varepsilon,1)$. Hence, by the mean value theorem, for any $x_1, x_2 \in [\varepsilon, 1]$, there exists $\xi \in (x_1, x_2)$ such that
    \begin{align}
        | f(x_1) - f(x_2) | = | x_1 - x_2 |\left| f'(\xi) \right| \text{.} \label{eq:bi-lipschitz-mvt}
    \end{align}
    Using that $|f'(\xi)| = \frac1\xi$ satisfies $\varepsilon \le f'(\xi) \le \frac1\varepsilon$ as $\varepsilon \le \xi \le 1$ yields the stated stated claim.
\end{proof}

\section{External Results}
\label{sec:external}
This section briefly summarizes external results that are necessary to prove our theorems.
Theorem~\ref{thm:dkw} states the Dvoretzky-Kiefer-Wolfowitz inequality, Assumption~\ref{def:label-gen} describes the candidate generation model used in Theorem~\ref{thm:sup-risk},
which relates the PLL risk~\eqref{eq:true-risk} to the risk in the supervised setting.
Theorem~\ref{thm:estimation-error} provides the estimation-error bound on which we build in our Lemma~\ref{lemma:main}. We recall Markov's inequality in Lemma~\ref{lemma:markov}.

\begin{theorem}[\mbox{\citealt[Dvoretzky-Kiefer-Wolfowitz Inequality]{DKW56,Naaman21}}]
    \label{thm:dkw}
    Let $(\Om, \sigmaal, \prob)$ be a probability space and $X, X_1, \dots, X_n \overset{i.i.d.}{\sim} \prob$ real-valued random variables on $\Om$.
    Then, for any $\delta \in (0, 1)$,
    \begin{align}
        \prob_X\!\left( \sup_{x \in \mathbb{R}} \left| \hat{F}_X(x) - F_X(x) \right| \leq \sqrt{\frac{\log(2 / \delta)}{2n}} \right) \geq 1 - \delta \text{,}
    \end{align}
    with $\hat{F}_X(x) = \frac{1}{n} \sum_{i=1}^{n} \mathds{1}_{\{ X_i \leq x \}}$ and $F_X(x) = \prob( X \leq x )$.
\end{theorem}

\begin{assumption}[\mbox{\citealt[Eq.~(5)]{FengL0X0G0S20}}]
    \label{def:label-gen}
    In the PLL setting (Section~\ref{sec:notations}),
    assume that $\prob_{XS}$ and $\prob_{XY}$ have Lebesgue densities $p_{XS}$ and $p_{XY}$, respectively,
    $p_{S \mid X, Y}(S) = p_{S \mid Y}(S)$,
    and the candidate generation model is of the form
    \begin{align}
        p_{XS}(X, S)
        = \sum_{y = 1}^k p_{S \mid Y = y}(S) p_{XY}(X, Y = y) \text{,} &  & \text{with }
                                                                       &  & p_{S \mid Y = y}(S) = \begin{cases}
                                                                                                      \frac{1}{2^{k-1} - 1} & \text{if } y \in S \text{,} \\
                                                                                                      0                     & \text{else.}
                                                                                                  \end{cases}
    \end{align}
\end{assumption}

The following theorem collects an identity by \citet{FengL0X0G0S20}.

\begin{theorem}[\mbox{\citealt[Eq.~(6), (7), and (8)]{FengL0X0G0S20}}]
    \label{thm:sup-risk}
    Let Assumption~\ref{def:label-gen} hold, $R(f)$ as in~\eqref{eq:true-risk}, and the true risk of the supervised classification setting $R_{\operatorname{sup}}(f) := \expected_{XY}[ \ell(f(X), Y) ]$. Then, $R_{\operatorname{sup}}(f) = \frac{1}{2} R(f)$.
\end{theorem}

\begin{theorem}[\mbox{\citealt[Theorem~4]{FengL0X0G0S20}}]
    \label{thm:estimation-error}
    Let $\ell : [0,1]^k \times \Y \to [0, M]$ be a bounded and $\lambda$-Lipschitz loss function in the first argument ($\lambda > 0$), that is, $\sup_{y \in \Y} | \ell(\mb{p}, y) - \ell(\mb{q}, y) | \leq \lambda \| \mb{p} - \mb{q} \|_2$ for $\mb{p}, \mb{q} \in [0, 1]^k$.
    Further, let $\Hy = \{ f : \X \to [0, 1]^k \mid f \text{ measurable, } \forall x \in \X : \sum_{j=1}^k f_j(x) = 1 \}$, $f^* = \arg\min_{f \in \Hy} R(f)$ be the true risk minimizer and $\hat{f} = \arg\min_{f \in \Hy} \hat{R}(f)$ be the empirical risk minimizer of the risks in~\eqref{eq:true-risk} and~\eqref{eq:erm}, respectively.
    Then, for any $\delta \in (0, 1)$, with $\prob_n$-probability of at least $1 - \delta$,
    \begin{align}
        R(\hat{f}) - R(f^*)
        \leq 4 \sqrt{2} \lambda \sum_{y=1}^k \mathfrak{R}_n( \Hy_y ) + C \sqrt{\frac{\log(2 / \delta)}{2n}} \text{,}
    \end{align}
    where $\mathfrak{R}_n( \Hy_y )$ is the empirical Rademacher complexity of $\Hy_y := \{ f_y \mid f \in \Hy \}$ and some constant $C > 0$.
    Further, using that $\mathfrak{R}_n( \Hy_y ) \leq C_{\Hy} / \sqrt{n}$ for some constants $C_{\Hy}, M > 0$, it holds with the same probability that
    \begin{align}
        R(\hat{f}) - R(f^*) \leq M \sqrt{\frac{\log(1 / \delta)}{n}} \text{.}
    \end{align}
\end{theorem}

\begin{lemma}[Markov inequality] \label{lemma:markov}
    For a real-valued random variable $X$ with probability distribution $\mathbb P$ and $a > 0$, it holds that
    \begin{align*}
        \mathbb P\left(|X| \ge a \right) \le \frac{\mathbb E\left(|X|\right)}{a}.
    \end{align*}
\end{lemma}

\section{Additional Setup}
\label{sec:hypar}

\begin{table}[t!]
    \caption{Overview of dataset characteristics grouped into real-world partially labeled datasets (top) and supervised multi-class classification datasets with added candidate labels (bottom).}
    \label{tab:datasets}
    \begin{center}
        \begin{small}
            \begin{tabular}{
                >{\raggedright\arraybackslash}m{0.13\linewidth}
                >{\centering\arraybackslash}m{0.13\linewidth}
                >{\centering\arraybackslash}m{0.13\linewidth}
                >{\centering\arraybackslash}m{0.13\linewidth}
                >{\centering\arraybackslash}m{0.13\linewidth}
                }
                \toprule
                {Dataset}         & {\#Instances $n$} & {\#Features $d$} & {\#Classes $k$} & {Avg.\ candidates} \\
                \midrule
                \emph{bird-song}  & \phantom{0}4\,966 & \phantom{0\,0}38 & \phantom{0}12   & 2.146              \\
                \emph{lost}       & \phantom{0}1\,122 & \phantom{0\,}108 & \phantom{0}14   & 2.216              \\
                \emph{mir-flickr} & \phantom{0}2\,778 & 1\,536           & \phantom{0}12   & 2.756              \\
                \emph{msrc-v2}    & \phantom{0}1\,755 & \phantom{0\,0}48 & \phantom{0}22   & 3.149              \\
                \emph{soccer}     & 17\,271           & \phantom{0\,}279 & 158             & 2.095              \\
                \emph{yahoo-news} & 22\,762           & \phantom{0\,}163 & 203             & 1.915              \\
                \midrule
                \emph{mnist}      & 70\,000           & \phantom{0\,}784 & \phantom{0}10   & 6.304              \\
                \emph{fmnist}     & 70\,000           & \phantom{0\,}784 & \phantom{0}10   & 5.953              \\
                \emph{kmnist}     & 70\,000           & \phantom{0\,}784 & \phantom{0}10   & 6.342              \\
                \emph{svhn}       & 99\,289           & 3\,072           & \phantom{0}10   & 4.878              \\
                \emph{cifar10}    & 60\,000           & 3\,072           & \phantom{0}10   & 1.900              \\
                \emph{cifar100}   & 60\,000           & 3\,072           & 100             & 1.399              \\
                \bottomrule
            \end{tabular}
        \end{small}
    \end{center}
\end{table}

In our experiments, we consider twelve datasets of which Table~\ref{tab:datasets} summarizes the characteristics.
As mentioned in Section~\ref{sec:methods}, we consider six state-of-the-art PLL approaches and our novel candidate cleaning technique.
We choose their parameters as recommended by the respective authors.
\begin{itemize}
    \item \textsc{Proden} \citep{LvXF0GS20}: For a fair comparison, we use the same base models for each particular dataset.
          For the colored-image datasets, we use a ResNet-9 architecture \citep{HeZRS16}.
          For all other image and non-image datasets, we use a standard $d$-300-300-300-$k$ MLP \citep{werbos1974beyond} with batch normalization \citep{IoffeS15} and ReLU activations \citep{GlorotBB11}.
          We choose the \emph{Adam} optimizer for training over a total of 200 epochs and use the one-cycle learning rate scheduler~\citep{smith2019super}.
          Also, we use mini-batched training with a batch size of 16 for the small-scale datasets (less than 5000 samples) and of 256 for the large-scale datasets (more than 5000 samples).
          This balances training duration and predictive quality.
    \item \textsc{Cc} \citep{FengL0X0G0S20}: We use the same base models and training procedures as mentioned above for \textsc{Proden}.
          Otherwise, there are no additional hyperparameters for \textsc{Cc}.
    \item \textsc{Valen} \citep{XuQGZ21}: We use the same base models and training procedures as mentioned above for \textsc{Proden}.
          Additionally, we use ten warm-up epochs and the three nearest neighbors to calculate the adjacency matrix.
    \item \textsc{Cavl} \citep{ZhangF0L0QS22}: We use the same base models and training procedures as mentioned above for \textsc{Proden}.
          Otherwise, there are no additional hyperparameters for \textsc{Cavl}.
    \item \textsc{Pop} \citep{0009LLQG23}: We use the same base models and training procedures as mentioned above for \textsc{Proden}.
          Also, we set $e_0 = 0.001$, $e_{end} = 0.04$, and $e_s = 0.001$.
          We abstain from using the data augmentations discussed in the paper for a fair comparison.
    \item \textsc{CroSel} \citep{crosel2024}: We use the same base models and training procedures as mentioned above for \textsc{Proden}.
          We use 10 warm-up epochs using \textsc{Cc} and $\lambda_{cr} = 2$.
          We abstain from using the data augmentations discussed in the paper for a fair comparison.
    \item \textsc{Conf+}Other method (our proposed approach): Our conformal candidate cleaning technique uses the same base models and training procedures as mentioned above for \textsc{Proden}.
          We use $R_{\operatorname{warmup}} = 10$ warm-up epochs, a validation set size of \SI{20}{\%}, and $\alpha_r = \frac{1}{n'} \sum_{i = 1}^{n'} \sum_{j \notin s_i} f_j(x_i)$.
          Otherwise, we use one of the given PLL classifiers for prediction-making.
\end{itemize}
We have implemented all approaches in \textsc{Python} using the \textsc{Pytorch} library.
Running all experiments requires approximately three days on a machine with 48 cores and one NVIDIA GeForce RTX 4090.
All our source code and data is available at \url{github.com/mathefuchs/pll-with-conformal-candidate-cleaning}.

\section{Additional Experiments}
\label{sec:more-exp}

\begin{table}[tbp]
    \caption{
        Average test-set accuracies ($\pm$ std.) on the real-world datasets.
        We benchmark our strategy (\textsc{Conf+}) as well as the cleaning method \textsc{Clsp} combined with all existing methods.
    }
    \label{tab:more-acc}
    \begin{center}
        \begin{small}
            \begin{tabular}{
                >{\raggedright\arraybackslash}m{0.145\linewidth}
                >{\centering\arraybackslash}m{0.112\linewidth}
                >{\centering\arraybackslash}m{0.112\linewidth}
                >{\centering\arraybackslash}m{0.112\linewidth}
                >{\centering\arraybackslash}m{0.112\linewidth}
                >{\centering\arraybackslash}m{0.112\linewidth}
                >{\centering\arraybackslash}m{0.112\linewidth}
                }
                \toprule
                Method                        & \emph{bird-song}          & \emph{lost}               & \emph{mir-flickr}         & \emph{msrc-v2}            & \emph{soccer}              & \emph{yahoo-news}         \\
                \midrule
                \mbox{\textsc{Proden} (2020)} & \mbox{75.55 ($\pm$ 1.08)} & \mbox{78.94 ($\pm$ 3.01)} & \mbox{67.05 ($\pm$ 1.18)} & \mbox{54.33 ($\pm$ 1.76)} & \mbox{54.18 ($\pm$ 0.55)}  & \mbox{65.25 ($\pm$ 1.00)} \\
                \mbox{\textsc{Clsp+Proden}}   & \mbox{74.61 ($\pm$ 0.84)} & \mbox{61.95 ($\pm$ 2.80)} & \mbox{60.53 ($\pm$ 2.95)} & \mbox{51.74 ($\pm$ 1.77)} & \mbox{31.93 ($\pm$ 29.15)} & \mbox{50.92 ($\pm$ 0.66)} \\
                \mbox{\textsc{Conf+P.}\ (no)} & \mbox{76.27 ($\pm$ 0.94)} & \mbox{79.56 ($\pm$ 1.96)} & \mbox{66.07 ($\pm$ 1.63)} & \mbox{53.00 ($\pm$ 2.24)} & \mbox{54.63 ($\pm$ 0.81)}  & \mbox{65.42 ($\pm$ 0.36)} \\
                \mbox{\textsc{Conf+Proden}}   & \mbox{76.99 ($\pm$ 0.90)} & \mbox{80.09 ($\pm$ 4.40)} & \mbox{66.91 ($\pm$ 1.57)} & \mbox{54.60 ($\pm$ 3.42)} & \mbox{54.77 ($\pm$ 0.84)}  & \mbox{65.93 ($\pm$ 0.42)} \\
                \midrule
                \mbox{\textsc{Cc} (2020)}     & \mbox{74.49 ($\pm$ 1.57)} & \mbox{78.23 ($\pm$ 2.11)} & \mbox{62.39 ($\pm$ 1.87)} & \mbox{50.96 ($\pm$ 2.03)} & \mbox{55.28 ($\pm$ 0.96)}  & \mbox{65.03 ($\pm$ 0.51)} \\
                \mbox{\textsc{Clsp+Cc}}       & \mbox{74.37 ($\pm$ 0.91)} & \mbox{60.88 ($\pm$ 3.71)} & \mbox{59.79 ($\pm$ 2.29)} & \mbox{49.64 ($\pm$ 2.06)} & \mbox{53.71 ($\pm$ 0.99)}  & \mbox{49.89 ($\pm$ 0.30)} \\
                \mbox{\textsc{Conf+Cc}}       & \mbox{75.01 ($\pm$ 1.84)} & \mbox{79.38 ($\pm$ 1.79)} & \mbox{63.37 ($\pm$ 0.45)} & \mbox{52.45 ($\pm$ 3.64)} & \mbox{55.52 ($\pm$ 0.74)}  & \mbox{64.35 ($\pm$ 0.64)} \\
                \midrule
                \mbox{\textsc{Valen} (2021)}  & \mbox{72.30 ($\pm$ 1.83)} & \mbox{70.18 ($\pm$ 3.44)} & \mbox{67.05 ($\pm$ 1.48)} & \mbox{49.20 ($\pm$ 1.37)} & \mbox{53.20 ($\pm$ 0.88)}  & \mbox{62.25 ($\pm$ 0.45)} \\
                \mbox{\textsc{Clsp+Valen}}    & \mbox{74.95 ($\pm$ 0.27)} & \mbox{59.03 ($\pm$ 2.67)} & \mbox{60.11 ($\pm$ 1.95)} & \mbox{49.92 ($\pm$ 1.80)} & \mbox{53.31 ($\pm$ 0.84)}  & \mbox{49.50 ($\pm$ 0.76)} \\
                \mbox{\textsc{Conf+Valen}}    & \mbox{71.22 ($\pm$ 1.03)} & \mbox{68.41 ($\pm$ 2.95)} & \mbox{61.61 ($\pm$ 2.79)} & \mbox{48.37 ($\pm$ 2.24)} & \mbox{52.49 ($\pm$ 1.00)}  & \mbox{62.16 ($\pm$ 0.74)} \\
                \midrule
                \mbox{\textsc{Cavl} (2022)}   & \mbox{69.78 ($\pm$ 3.00)} & \mbox{72.12 ($\pm$ 1.08)} & \mbox{65.02 ($\pm$ 1.34)} & \mbox{52.67 ($\pm$ 2.32)} & \mbox{55.06 ($\pm$ 0.48)}  & \mbox{61.91 ($\pm$ 0.46)} \\
                \mbox{\textsc{Clsp+Cavl}}     & \mbox{73.13 ($\pm$ 1.23)} & \mbox{58.76 ($\pm$ 1.75)} & \mbox{59.86 ($\pm$ 2.92)} & \mbox{48.65 ($\pm$ 2.31)} & \mbox{53.48 ($\pm$ 0.76)}  & \mbox{49.48 ($\pm$ 0.37)} \\
                \mbox{\textsc{Conf+Cavl}}     & \mbox{72.00 ($\pm$ 1.22)} & \mbox{71.24 ($\pm$ 3.81)} & \mbox{64.42 ($\pm$ 0.89)} & \mbox{51.63 ($\pm$ 5.03)} & \mbox{54.85 ($\pm$ 0.92)}  & \mbox{62.43 ($\pm$ 0.43)} \\
                \midrule
                \mbox{\textsc{Pop} (2023)}    & \mbox{75.17 ($\pm$ 1.04)} & \mbox{77.79 ($\pm$ 2.11)} & \mbox{67.93 ($\pm$ 1.44)} & \mbox{53.83 ($\pm$ 0.69)} & \mbox{55.31 ($\pm$ 0.71)}  & \mbox{65.09 ($\pm$ 0.64)} \\
                \mbox{\textsc{Clsp+Pop}}      & \mbox{74.25 ($\pm$ 0.89)} & \mbox{60.18 ($\pm$ 2.48)} & \mbox{59.61 ($\pm$ 1.84)} & \mbox{50.58 ($\pm$ 1.47)} & \mbox{32.08 ($\pm$ 29.29)} & \mbox{50.77 ($\pm$ 0.42)} \\
                \mbox{\textsc{Conf+Pop}}      & \mbox{77.58 ($\pm$ 1.01)} & \mbox{78.41 ($\pm$ 2.13)} & \mbox{66.21 ($\pm$ 2.19)} & \mbox{54.82 ($\pm$ 3.60)} & \mbox{56.49 ($\pm$ 1.10)}  & \mbox{65.25 ($\pm$ 0.23)} \\
                \midrule
                \mbox{\textsc{CroSel} (2024)} & \mbox{75.11 ($\pm$ 1.79)} & \mbox{81.24 ($\pm$ 3.68)} & \mbox{67.58 ($\pm$ 1.16)} & \mbox{52.23 ($\pm$ 2.83)} & \mbox{52.64 ($\pm$ 1.21)}  & \mbox{67.72 ($\pm$ 0.32)} \\
                \mbox{\textsc{Clsp+CroSel}}   & \mbox{76.53 ($\pm$ 1.34)} & \mbox{63.72 ($\pm$ 2.23)} & \mbox{59.75 ($\pm$ 2.79)} & \mbox{51.29 ($\pm$ 1.69)} & \mbox{52.24 ($\pm$ 0.84)}  & \mbox{53.53 ($\pm$ 0.93)} \\
                \mbox{\textsc{Conf+CroSel}}   & \mbox{77.76 ($\pm$ 0.50)} & \mbox{81.15 ($\pm$ 2.57)} & \mbox{65.93 ($\pm$ 1.94)} & \mbox{54.10 ($\pm$ 2.75)} & \mbox{54.97 ($\pm$ 0.65)}  & \mbox{67.55 ($\pm$ 0.22)} \\
                \bottomrule
            \end{tabular}
        \end{small}
    \end{center}
\end{table}

\begin{table}[tbp]
    \caption{
        Average test-set accuracies ($\pm$ std.) on the supervised datasets with added incorrect candidate labels.
        We benchmark our strategy (\textsc{Conf}) as well as the cleaning method \textsc{Clsp} combined with all existing methods.
        We use the pre-trained \textsc{Blip-2} model for all results in this table.
    }
    \label{tab:more-acc-2}
    \begin{center}
        \begin{small}
            \begin{tabular}{
                >{\raggedright\arraybackslash}m{0.19\linewidth}
                >{\centering\arraybackslash}m{0.13\linewidth}
                >{\centering\arraybackslash}m{0.13\linewidth}
                >{\centering\arraybackslash}m{0.13\linewidth}
                >{\centering\arraybackslash}m{0.13\linewidth}
                >{\centering\arraybackslash}m{0.13\linewidth}
                }
                \toprule
                Method                                  & \emph{mnist}               & \emph{fmnist}             & \emph{kmnist}             & \emph{cifar10}             & \emph{cifar100}           \\
                \midrule
                \mbox{\textsc{Proden}}                  & \mbox{87.21 ($\pm$ 0.83)}  & \mbox{71.18 ($\pm$ 2.95)} & \mbox{59.31 ($\pm$ 1.22)} & \mbox{99.07 ($\pm$ 0.05)}  & \mbox{90.51 ($\pm$ 0.18)} \\
                \mbox{\textsc{Clsp+Proden}}             & \mbox{85.91 ($\pm$ 2.42)}  & \mbox{72.11 ($\pm$ 2.81)} & \mbox{62.61 ($\pm$ 1.00)} & \mbox{99.03 ($\pm$ 0.07)}  & \mbox{90.31 ($\pm$ 0.13)} \\
                \mbox{\textsc{Conf+P.} (no correction)} & \mbox{91.74 ($\pm$ 0.34)}  & \mbox{78.38 ($\pm$ 0.50)} & \mbox{66.88 ($\pm$ 0.76)} & \mbox{99.03 ($\pm$ 0.04)}  & \mbox{91.16 ($\pm$ 0.10)} \\
                \mbox{\textsc{Conf+Proden}}             & \mbox{91.55 ($\pm$ 0.23)}  & \mbox{78.09 ($\pm$ 0.33)} & \mbox{66.43 ($\pm$ 0.38)} & \mbox{99.03 ($\pm$ 0.04)}  & \mbox{91.16 ($\pm$ 0.10)} \\
                \midrule
                \mbox{\textsc{Cc}}                      & \mbox{86.29 ($\pm$ 2.18)}  & \mbox{66.19 ($\pm$ 2.77)} & \mbox{58.29 ($\pm$ 0.32)} & \mbox{99.07 ($\pm$ 0.05)}  & \mbox{73.43 ($\pm$ 1.40)} \\
                \mbox{\textsc{Clsp+Cc}}                 & \mbox{85.46 ($\pm$ 1.93)}  & \mbox{71.37 ($\pm$ 2.34)} & \mbox{61.37 ($\pm$ 1.09)} & \mbox{99.03 ($\pm$ 0.06)}  & \mbox{89.00 ($\pm$ 1.56)} \\
                \mbox{\textsc{Conf+Cc}}                 & \mbox{85.20 ($\pm$ 4.16)}  & \mbox{59.75 ($\pm$ 2.68)} & \mbox{57.07 ($\pm$ 0.66)} & \mbox{99.04 ($\pm$ 0.03)}  & \mbox{71.45 ($\pm$ 0.94)} \\
                \midrule
                \mbox{\textsc{Valen}}                   & \mbox{78.91 ($\pm$ 0.80)}  & \mbox{66.53 ($\pm$ 2.65)} & \mbox{58.48 ($\pm$ 0.45)} & \mbox{92.17 ($\pm$ 0.54)}  & \mbox{67.24 ($\pm$ 2.49)} \\
                \mbox{\textsc{Clsp+Valen}}              & \mbox{84.72 ($\pm$ 3.10)}  & \mbox{68.84 ($\pm$ 1.49)} & \mbox{60.76 ($\pm$ 0.76)} & \mbox{98.17 ($\pm$ 0.17)}  & \mbox{84.53 ($\pm$ 1.23)} \\
                \mbox{\textsc{Conf+Valen}}              & \mbox{74.20 ($\pm$ 21.99)} & \mbox{69.09 ($\pm$ 2.71)} & \mbox{60.95 ($\pm$ 2.59)} & \mbox{42.63 ($\pm$ 19.92)} & \mbox{60.44 ($\pm$ 1.91)} \\
                \midrule
                \mbox{\textsc{Cavl}}                    & \mbox{71.11 ($\pm$ 3.92)}  & \mbox{59.85 ($\pm$ 6.49)} & \mbox{48.15 ($\pm$ 5.07)} & \mbox{41.78 ($\pm$ 21.40)} & \mbox{31.95 ($\pm$ 1.80)} \\
                \mbox{\textsc{Clsp+Cavl}}               & \mbox{83.72 ($\pm$ 3.57)}  & \mbox{67.38 ($\pm$ 2.59)} & \mbox{62.06 ($\pm$ 2.12)} & \mbox{87.34 ($\pm$ 12.71)} & \mbox{68.02 ($\pm$ 1.96)} \\
                \mbox{\textsc{Conf+Cavl}}               & \mbox{71.86 ($\pm$ 4.57)}  & \mbox{59.54 ($\pm$ 6.62)} & \mbox{52.14 ($\pm$ 3.89)} & \mbox{29.97 ($\pm$ 15.73)} & \mbox{37.34 ($\pm$ 2.59)} \\
                \midrule
                \mbox{\textsc{Pop}}                     & \mbox{87.08 ($\pm$ 0.58)}  & \mbox{72.30 ($\pm$ 2.63)} & \mbox{60.63 ($\pm$ 1.15)} & \mbox{99.06 ($\pm$ 0.04)}  & \mbox{90.50 ($\pm$ 0.21)} \\
                \mbox{\textsc{Clsp+Pop}}                & \mbox{85.43 ($\pm$ 2.60)}  & \mbox{72.05 ($\pm$ 2.41)} & \mbox{62.49 ($\pm$ 0.90)} & \mbox{99.04 ($\pm$ 0.07)}  & \mbox{90.37 ($\pm$ 0.06)} \\
                \mbox{\textsc{Conf+Pop}}                & \mbox{91.19 ($\pm$ 0.29)}  & \mbox{79.15 ($\pm$ 1.23)} & \mbox{67.37 ($\pm$ 0.28)} & \mbox{99.05 ($\pm$ 0.04)}  & \mbox{91.12 ($\pm$ 0.09)} \\
                \midrule
                \mbox{\textsc{CroSel}}                  & \mbox{91.84 ($\pm$ 0.44)}  & \mbox{76.34 ($\pm$ 1.21)} & \mbox{65.55 ($\pm$ 0.81)} & \mbox{99.07 ($\pm$ 0.02)}  & \mbox{75.86 ($\pm$ 2.26)} \\
                \mbox{\textsc{Clsp+CroSel}}             & \mbox{91.70 ($\pm$ 0.62)}  & \mbox{74.42 ($\pm$ 1.02)} & \mbox{67.93 ($\pm$ 1.07)} & \mbox{99.08 ($\pm$ 0.02)}  & \mbox{88.80 ($\pm$ 0.85)} \\
                \mbox{\textsc{Conf+CroSel}}             & \mbox{91.85 ($\pm$ 0.61)}  & \mbox{77.31 ($\pm$ 0.46)} & \mbox{64.73 ($\pm$ 1.52)} & \mbox{99.07 ($\pm$ 0.03)}  & \mbox{77.26 ($\pm$ 0.98)} \\
                \bottomrule
            \end{tabular}
        \end{small}
    \end{center}
\end{table}

\begin{table}[tbp]
    \caption{
        Number of significant differences aggregated from Table~\ref{tab:more-acc} and~\ref{tab:more-acc-3} using a paired t-test (level \SI{5}{\%}).
    }
    \label{tab:more-acc-3}
    \begin{center}
        \begin{small}
            \begin{tabular}{
                >{\raggedright\arraybackslash}m{0.3\linewidth}
                >{\centering\arraybackslash}m{0.07\linewidth}
                >{\centering\arraybackslash}m{0.08\linewidth}
                >{\centering\arraybackslash}m{0.09\linewidth}
                }
                \toprule
                Comparison vs. all others            & Wins         & Ties & Losses       \\
                \midrule
                \textsc{Proden}                      & 27           & 37   & \phantom{0}8 \\
                \textsc{Clsp+Proden}                 & 18           & 27   & 27           \\
                \textsc{Conf+Proden} (no correction) & 41           & 25   & \phantom{0}6 \\
                \textsc{Conf+Proden}                 & 50           & 20   & \phantom{0}2 \\
                \midrule
                \textsc{Cc}                          & 18           & 39   & 15           \\
                \textsc{Clsp+Cc}                     & 18           & 22   & 32           \\
                \textsc{Conf+Cc}                     & 22           & 30   & 20           \\
                \midrule
                \textsc{Valen}                       & \phantom{0}5 & 30   & 37           \\
                \textsc{Clsp+Valen}                  & 17           & 15   & 40           \\
                \textsc{Conf+Valen}                  & \phantom{0}5 & 24   & 43           \\
                \midrule
                \textsc{Cavl}                        & \phantom{0}5 & 26   & 41           \\
                \textsc{Clsp+Cavl}                   & \phantom{0}9 & 19   & 44           \\
                \textsc{Conf+Cavl}                   & \phantom{0}4 & 26   & 42           \\
                \midrule
                \textsc{Pop}                         & 29           & 39   & \phantom{0}4 \\
                \textsc{Clsp+Pop}                    & 17           & 23   & 32           \\
                \textsc{Conf+Pop}                    & 49           & 22   & \phantom{0}1 \\
                \midrule
                \textsc{CroSel}                      & 30           & 33   & \phantom{0}9 \\
                \textsc{Clsp+CroSel}                 & 25           & 15   & 32           \\
                \textsc{Conf+CroSel}                 & 44           & 22   & \phantom{0}6 \\
                \bottomrule
            \end{tabular}
        \end{small}
    \end{center}
\end{table}

In addition to our cleaning method (\textsc{Conf}), we also benchmark the existing cleaning method \textsc{Clsp} \citep{0001WY024} on all datasets in Table~\ref{tab:more-acc}, \ref{tab:more-acc-2}, and~\ref{tab:more-acc-3}, similar to the experiments in Section~\ref{sec:exp}.
Instead of training a ResNet-9 base model from scratch as done in Section~\ref{sec:methods}, we use the pre-trained \textsc{Blip-2} model~\citep{0008LSH23} for the experiments in Table~\ref{tab:more-acc-2} below.
We repeat all experiments five times and report means and standard deviations.

We observe that the \textsc{Clsp} models perform well on the image datasets (e.g., \emph{cifar100}) but poorly on the real-world tabular PLL datasets shown in Table~\ref{tab:more-acc}.
We attribute this to the fact that \textsc{Clsp} relies on the latent representation of large-scale vision models.
In contrast, our method \textsc{Conf} gives strong results on, both, real-world and image data.
This hypothesis is supported by Table~\ref{tab:more-acc-3}:
The approaches \textsc{Conf+Proden}, \textsc{Conf+Pop}, and \textsc{Conf+CroSel} that combine the respective approaches with our candidate cleaning strategy win most frequently.

\end{document}